\documentclass{article}

\usepackage[main, final]{neurips_2026}

\usepackage[utf8]{inputenc} % allow utf-8 input
\usepackage[T1]{fontenc}    % use 8-bit T1 fonts
\usepackage{hyperref}       % hyperlinks
\usepackage{url}            % simple URL typesetting
\usepackage{booktabs}       % professional-quality tables
\usepackage{amsfonts}       % blackboard math symbols
\usepackage{nicefrac}       % compact symbols for 1/2, etc.
\usepackage{microtype}      % microtypography
\usepackage{xcolor}         % colors

\usepackage{graphicx}
\usepackage{enumitem}
\setlist[enumerate]{label=(\roman*), leftmargin=*, itemsep=0pt, topsep=2pt}
\RequirePackage{algorithm}
\RequirePackage{algorithmic}
\usepackage{pifont} % Required for the symbols
\usepackage{caption}
\usepackage{float}
\usepackage{dsfont}
\usepackage{multirow} 
\usepackage{wrapfig}

\usepackage{amsmath}
\usepackage{amssymb}
\usepackage{mathtools}
\usepackage{amsthm}
\usepackage[capitalize,noabbrev]{cleveref}
\newtheoremstyle{mainplain}% name
  {0.5em}% space above
  {0.5em}% space below
  {\itshape}% body font
  {}% indent amount
  {\bfseries}% theorem head font
  {.}% punctuation after head
  {.5em}% space after head
  {}% head spec
\newtheoremstyle{maindefinition}% name
  {0.5em}% space above
  {0.5em}% space below
  {}% body font
  {}% indent amount
  {\bfseries}% theorem head font
  {.}% punctuation after head
  {.5em}% space after head
  {}% head spec
\newtheoremstyle{mainremark}% name
  {0.5em}% space above
  {0.5em}% space below
  {}% body font
  {}% indent amount
  {\itshape}% theorem head font
  {.}% punctuation after head
  {.5em}% space after head
  {}% head spec
\theoremstyle{mainplain}
\newtheorem{theorem}{Theorem}[section]
\newtheorem*{theorem*}{Theorem}
\newtheorem{proposition}[theorem]{Proposition}
\newtheorem*{proposition*}{Proposition}
\newtheorem{lemma}[theorem]{Lemma}
\newtheorem{corollary}[theorem]{Corollary}
\theoremstyle{maindefinition}
\newtheorem{definition}[theorem]{Definition}
\newtheorem{assumption}[theorem]{Assumption}
\newtheorem{example}[theorem]{Example}
\newtheorem*{examplecont}{Example~\ref{ex:vp-matched}, continued}
\theoremstyle{mainremark}
\newtheorem{remark}[theorem]{Remark}
\newcommand{\R}{\mathbb{R}}
\newcommand{\E}{\mathbb{E}}
\newcommand{\Var}{\operatorname{Var}}
\newcommand{\cN}{\mathcal{N}}
\newcommand{\cA}{\mathcal{A}}
\newcommand{\cV}{\mathcal{V}}
\newcommand{\backvec}[1]{\reflectbox{$\vec{\reflectbox{$#1$}}$}}
\newcommand{\dd}{\mathrm{d}}
\newcommand{\given}{\,|\,}

\title{Sticky Jump Diffusions: A Unifying View of Masked, Continuous, and Hybrid Diffusion}

\author{%
  Pascal Jutras-Dubé \\%\thanks{Use footnote for providing further information about author (webpage, alternative address)} \\
  Department of Computer Science\\
  Purdue University\\
  \texttt{pjutrasd@purdue.edu} \\
  \And
  Patrick Pynadath \\
  Department of Computer Science\\
  Purdue University\\
  \texttt{ppynadat@purdue.edu} \\
  \AND
  Jeremy Lu \\
  Department of Computer Science\\
  Purdue University\\
  \texttt{lu1008@purdue.edu} \\
  \And
  Yuan Gao \\
  Department of Mathematics \\
  Purdue University\\
  \texttt{gao662@purdue.edu} \\
  \And
  Ruqi Zhang \\
  Department of Computer Science\\
  Purdue University\\
  \texttt{ruqiz@purdue.edu} \\
}

\begin{document}

\maketitle

\begin{abstract}
We introduce \emph{sticky jump diffusions} (SJDs), continuous-time Markov processes on $\R^d$ whose discrete anchors are token embeddings. In forward time, anchors release their mass at a hazard rate and the released mass diffuses in the continuous ambient space; time reversal couples a score-driven SDE with a sticky jump kernel whose rate and destination are fixed by flux balance with the forward law. We estimate the score and the per-anchor reverse hazards from a single denoising classifier via \emph{denoising hazard matching}, the hazard analogue of denoising score matching, with simulation-free cross-entropy training. SJD recovers masked diffusion, continuous diffusion, and hybrid diffusion as limits. Its reversal explains features that each family treats as given: the mask of masked diffusion carries no evidence about the source token because the unsticking kernel of every anchor collapses to the same absorbing point; the terminal projection of continuous diffusion is required due to the absence of atoms in its forward marginal, without which flux balance yields no reverse jumps; and the update rules of hybrid diffusion (commit rate, destination, and drift) all follow from flux balance rather than from separate design. Beyond these limits, the unsticking kernel becomes a design space: a cross-position blending corrupts each position toward a blend of its neighbors' clean values or embeddings, turning dependency structure such as spatial locality or a constraint graph into an inductive bias of the corruption itself, and improves over the identity-kernel hybrid on CIFAR-10, Text8, and Sudoku. Our code is available at \url{https://github.com/PascalJD/sticky-jump-diffusions}.
\end{abstract}

\section{Introduction}\label{sec:introduction}
Diffusion-based generative modeling has been extended from continuous data \citep{sohl2015thermo,song2019smld,ho2020ddpm,song2021sde,karras2022edm} to discrete sequences such as text. Two families dominate the discrete setting. \emph{Masked diffusion} operates on the token space and outputs valid tokens by construction, but each position is mapped to a single absorbing token, and the per-position factorization erases graded information about how close the corruption is to a candidate token \citep{austin2021d3pm,sahoo2024mdlm,shi2024md4,lou2024sedd}. \emph{Continuous diffusion} operates on token embeddings and uses standard score-based learning, but its terminal state lies off the discrete support and must be mapped back by a projection step external to the dynamics \citep{li2022diffusionlm,dieleman2022cdcd,chen2023analogbits}. Recent \emph{hybrid diffusion} pairs the two: a masked chain decides which positions are committed, and a continuous latent carries graded information for the masked ones \citep{zheng2025cadd,pynadath2025candi,zhou2025ccdd}. A review of these three families is given in \Cref{app:prelim}.

We define \emph{sticky jump diffusions} (SJDs), a class of continuous-time Markov processes on a continuous state space $\mathsf X = \R^d$ with a finite or countable set of anchors $\cA \subset \R^d$ identified with token embeddings. All three families above arise as limits of SJDs. In forward time, data start on anchors; an anchor releases its mass at a hazard rate $\vec\lambda_t$ into the continuous region $\mathsf X_\cA$, where the trajectory diffuses by a standard SDE. The marginal $p_t$ splits into a continuous density on $\mathsf X_\cA$ and atoms on $\cA$, generalizing the classical sticky-boundary diffusions of a single absorbing point on the half-line \citep{feller1952parabolic,feller1954diffusion, ito1963brownianhalfline} to anchor sets of vocabulary size living in $\R^d$. \Cref{fig:diffusion_comparison} visually compares SJD with masked and continuous diffusion.

Time reversal of an SJD produces a generative process with two coupled mechanisms: a score-driven SDE on $\mathsf X_\cA$ and a sticky jump kernel that commits mass to anchors. The per-anchor reverse intensities are fixed by flux balance \citep{conforti2022timereversal}, so the choice of where to commit and at what rate is determined by the forward law. Token commitment is therefore a reverse-time jump event of the stochastic process, with a rate and a destination that the dynamics specify, rather than a hand-designed schedule applied at sampling time.

To turn the reverse-time identity into a usable generative model, we introduce \emph{denoising hazard matching} (DHM), the hazard analogue of denoising score matching \citep{vincent2011denoisingscorematching}: the reverse hazard equals the conditional expectation of a forward-known target under the SJD corruption. Combined with a denoising classifier $P_\theta$, DHM yields both the score and the per-anchor reverse intensities from the same network, by an analytic reweighting of $P_\theta$ that is Rao–Blackwellized and needs no separate hazard head. Training is plain cross-entropy on SJD-corrupted pairs and is simulation-free under standard conditions on the forward dynamics.

\begin{figure*}[t]
  \centering
  \includegraphics[width=\textwidth]{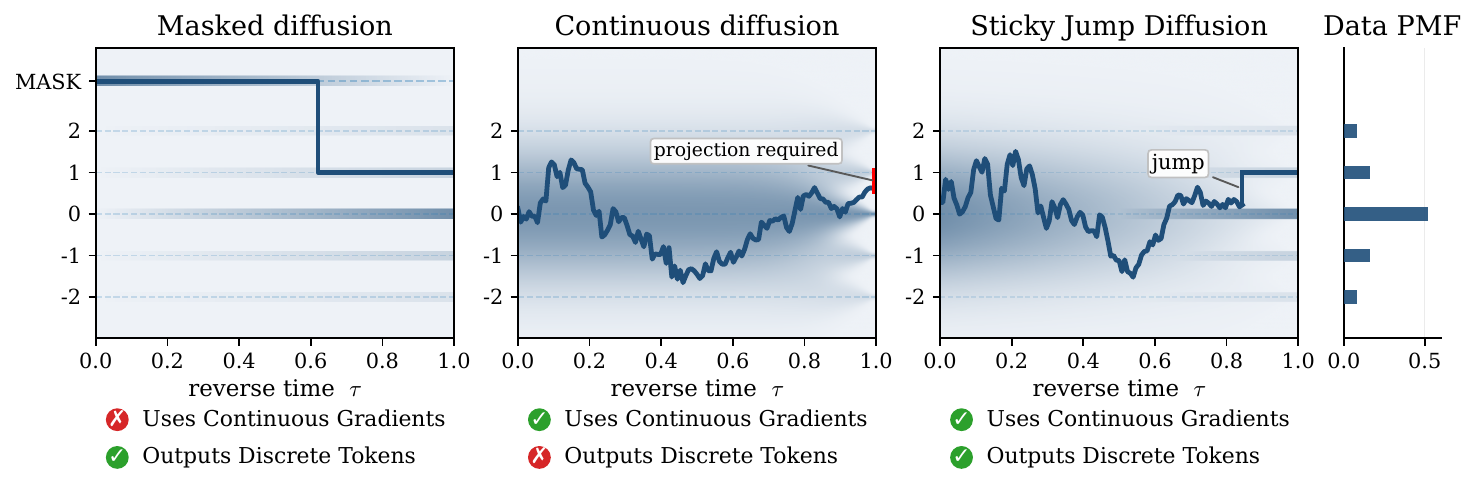}
  \caption{Reverse-time trajectories over each model's path space. Masked diffusion is discrete but gradient-free; continuous diffusion is gradient-based but requires a terminal projection; SJD couples ambient gradients with sticky jumps onto anchors.}
  \label{fig:diffusion_comparison}
\end{figure*}
The same construction places existing model classes inside SJD, and each limit inherits an explanation from the reversal. In masked diffusion, the unsticking kernel collapses to the absorbing token, and commitment is forced to be time-only. In hybrid diffusion, the kernel matches the forward perturbation kernel: commit rate, destination, and drift follow from flux balance. In continuous diffusion, there is no sticking jump: without atoms in the forward marginal, there are no reverse jumps, so a terminal projection outside the dynamics is needed. The intermediate regime is the one we study empirically: the unsticking kernel there carries a cross-position blending toward a weighted combination of neighbors' clean values or embeddings, which improves over the identity-kernel hybrid on CIFAR-10, Text8, and Sudoku.

\paragraph{Contributions.}
\begin{enumerate}
  \item \textbf{The SJD process class.} We introduce sticky jump diffusions, continuous-time jump-diffusions on $\R^d$ whose forward marginal splits into a density on a continuous off-anchor region and atoms on a discrete anchor set identified with token embeddings.

  \item \textbf{Reverse dynamics from a single classifier.} We derive the time reversal process of SJD: a score-driven SDE coupled with a sticky jump kernel whose per-anchor rates are fixed by flux balance. Both the score and the reverse hazard are estimated by a denoising classifier through denoising hazard matching, the hazard analogue of denoising score matching, an analytic Rao-Blackwellized reweighting that needs no second network and trains by cross-entropy.

  \item \textbf{A unifying view of masked, continuous, and hybrid diffusion.} SJD recovers absorbing-state masked diffusion, hybrid diffusion, and continuous diffusion as limits. Where these families assemble the reverse sampler by hand, SJD makes the discrete and continuous coordinates one reversed state: commit rate, destination token, and drift all follow from the forward corruption by flux balance.

  \item \textbf{The jump kernel yields a new design space.} We empirically study cross-position jump kernels: each position's corruption is centered on a blend of the other positions' clean values or embeddings rather than its own embedding alone. We instantiate this on three data types whose structure it can express, encoding spatial locality on CIFAR-10, local sequential correlation on Text8, and the row/column/box constraint graph on Sudoku; it improves FID, valid-word generation, and board completion over the identity-kernel hybrid under matched settings.
\end{enumerate}

\section{Sticky jump diffusion (SJD) models}\label{sec:sjd}
Let the state space be $\mathsf{X}=\R^d$ and let $\cA=\{a_k\}_{k=1}^K\subset\R^d$ be a finite (or countable) set of anchors.  We write $\mathsf{X}_\cA := \R^d\setminus\cA$ for the continuous region and allow the law $p_t$ of $X_t$ to have an absolutely continuous part on $\mathsf{X}_\cA$ and atoms on $\cA$. Notation and assumptions are introduced throughout the text and reiterated in \cref{app:notation}. Proofs are in \cref{app:proofs}.

\subsection{Forward noising dynamics: unsticking}\label{forward-sjd}
SJDs are defined for a single particle in what follows; a running example lifts the construction to sequences and is carried throughout the paper.

\begin{definition}[Forward (noising) process]\label{def:forward}
Fix a terminal time $T>0$. The forward process $(\vec X_t)_{t\in[0,T]}$ starts at an anchor, $\vec X_0 = a \in \cA$ with $a\sim p_{\mathrm{data}}$, and evolves as follows. While stuck at its anchor, $\vec X_t = a$, it waits and
\begin{equation}\label{eq:forward-jump}
    \text{with hazard } \vec\lambda_t,\ \text{ unsticks to a new state } Y\sim R_t(\cdot\given a).
\end{equation}
Here $R_t(\cdot\given a)$ is a fixed \emph{unsticking} kernel supported on $\mathsf{X}_\cA$; when it is absolutely continuous we write $R_t(\dd y\given a)=r_t(y\given a)\dd y$.
Once $\vec X_t \in \mathsf{X}_\cA$, it diffuses according to:
\begin{equation}\label{eq:forwardSDE}
    \dd \vec X_t = \vec b_t(\vec X_t)\dd t + g_t\dd \vec W_t.
\end{equation}
The diffusion coefficient $g_t>0$ is a scalar, $\vec b_t$ is a drift, $\dd \vec W_t$ is a Wiener increment. There are no forward jumps into anchors, so once a trajectory unsticks, it never returns to $\cA$.
\end{definition}

In words: data start at anchors, wait at their anchor for a hazard-governed amount of time, then unstick into the continuous region and diffuse. Intuitively, the forward process transports mass from the discrete anchor set into the continuous ambient space. The third panel of \Cref{fig:diffusion_comparison} shows an SJD trajectory.

The following construction is used in all of our experiments. It introduces two parameters: a blending matrix $W$, the axis our experiments vary, and a variance ratio $\eta$, which we keep at $\eta=1$, the width of the hybrid perturbation kernel (\Cref{sec:conn-hybrid}).

\begin{example}[VP-matched sequence SJD with non-local jump kernel]\label{ex:vp-matched}
In practice, $\vec X_0$ is typically a sequence of anchors $\vec X_0 = (a^{(1)}, \dots, a^{(L)}) \in \cA^L$ (e.g.\ a sentence). The unsticking kernel at a position may depend on the whole clean sequence, not on that position's anchor alone.  We realize this here through a blending matrix $W\in\R^{L\times L}$ that couples the positions. With VP coefficients $\alpha(t),\sigma(t)$ and an embedding $E:\cA\to\R^d$, define the \emph{blended mean} at position $i$,
\begin{equation}\label{eq:blended-mean}
    \mu_i(\vec X_0) := \big(W E(\vec X_0)\big)_i = \sum_{j=1}^L W_{ij}\, E(a^{(j)}),
\end{equation}
where $E(\vec X_0)\in\R^{L\times d}$ stacks the per-position embeddings. Position $i$ is corrupted toward $\mu_i$ rather than its own embedding only. At $W=I$ we recover $\mu_i = E(a^{(i)})$ and the positions' jumps are independent. Position $i$ then unsticks to
\begin{equation}\label{eq:nonlocal-r}
    r_t\big(y\given i, \vec X_0\big) = \cN\Big(y;\alpha(t)\mu_i(\vec X_0),\eta^2\sigma^2(t)I\Big),\qquad \eta\in(0,1].
\end{equation}
The single-particle kernel $r_t(\cdot\given a)$ of \Cref{def:forward} has become $r_t(\cdot\given i, \vec X_0)$, now depending on the whole initial sequence. It departs from the usual VP perturbation kernel in its mean, set by $W$, and the variance, set by $\eta\in(0,1]$. \Cref{fig:eta-commit} shows the effect of $\eta$. 
\end{example}

\begin{remark}[Mass split]\label{rem:mass-split}
At each $t$, the law $p_t$ decomposes as $p_t=p_t^{\mathrm{ac}}+p_t^{\mathrm{at}}$, where $p_t^{\mathrm{ac}}$ is an absolutely continuous density on $\mathsf{X}_\cA$ and $p_t^{\mathrm{at}}=\sum_{a\in\cA} m_t(a)\delta_{a}$ is an atomic density with anchor mass $m_t(a)\in[0,1]$. The anchor mass $m_t(a)$ decays deterministically given $\vec\lambda_t$ when the diffusion does not re-attach to anchors in forward time.
\end{remark}

\begin{remark}[Conditional law is a mixture over unstick times]\label{rem:mixture}
Conditioned on $\vec X_0 = a$, a particle that has already unstuck by time $t$ did so at some random time $\tau\in(0,t)$, and its continuous density at $t$ is the convolution of the unsticking kernel at $\tau$ with the diffusion semigroup from $\tau$ to $t$. Integrating over $\tau$ against the hazard-survival density gives
\[
p_t^{\mathrm{ac}}(y\mid a) = \int_0^t \vec\lambda_\tau(a)S_a(\tau)\big(r_\tau(\cdot\mid a) * K_{\tau\to t}\big)(y)\dd\tau,
\]
where $S_a(\tau)=\exp\big(-\int_0^\tau \vec\lambda_s(a)\dd s\big)$ is the anchor survival function and $K_{\tau\to t}$ is the forward diffusion semigroup. In general, $p_t^{\mathrm{ac}}(\cdot\mid a)$ is therefore a mixture and not a single Gaussian, even when both $r_\tau(\cdot\mid a)$ and $K_{\tau\to t}$ are. Two particles that started at the same anchor but unstuck at different times carry different accumulated variances, and $p_t^{\mathrm{ac}}(\cdot\mid a)$ averages over these.
\end{remark}

\begin{remark}[Terminal prior]\label{rem:prior}
In \cref{ex:vp-matched}, when the VP coefficients $\alpha(T)\to 0$ and $\sigma(T)\to 1$ and every anchor loses its mass by $T$ (i.e.\ $S_a(T)=0$ for all $a$, so $p_T^{\mathrm{at}}=0$), every mixture component of \Cref{rem:mixture} converges to $\mathcal N(0, I_d)$ \emph{regardless of $\tau$, $\eta$, or the blending matrix $W$}.
\end{remark}

\subsection{Reverse-time generative dynamics: sticking}\label{backward-sjd}
Intuitively, if the forward process jumps off anchors, then the reverse-time process must jump into anchors. We state this formally below.

\begin{theorem}[Time reversal for jump-diffusions]\label{thm:anchored-time-reversal}
Under mild regularity assumptions, the reverse-time process $(\backvec{X}_\tau)_{\tau\in[0,T]}$ with $\tau=T-t$ is Markov on $\R^d$ and admits:
\begin{align}
\dd \backvec X_\tau &= \big(\backvec b_t(\backvec X_\tau)\big)\dd \tau + g_t \dd \backvec W_\tau \quad \text{on }\mathsf{X}_\cA, \label{eq:backwardSDE}\\
\backvec b_t(x) &= -\vec b_t(x) + g_t^2 \nabla \log p_t(x), \quad x\in\mathsf{X}_\cA, \label{eq:score-drift}
\end{align}
and a jump kernel that sticks to anchors, uniquely defined by the flux equation
\begin{equation}
p_t(\dd y) \backvec J_t(y,\dd x) = p_t(\dd x) \vec J_t(x,\dd y).
\label{eq:flux}
\end{equation}
In particular, for $y\in\mathsf{X}_\cA$ and $x=a\in\cA$,
\begin{equation}
\backvec J_t(y,\{a\}) = \frac{p_t(\{a\})}{p_t(y)}\vec\lambda_t r_t(y\given a).
\label{eq:sticky-rate}
\end{equation}
\end{theorem}

\section{Learning sticky dynamics}\label{sec:hazard}
To turn \Cref{thm:anchored-time-reversal} into a usable generative model we need to evaluate, at every $(y, t)$, the reverse drift \eqref{eq:score-drift} and the sticky-jump rates \eqref{eq:sticky-rate}. Both involve the SJD marginal $p_t$, which we cannot evaluate in closed form: the off-anchor conditional $p_t^{\mathrm{ac}}(\cdot\given a)$ is the mixture over unstick times of \Cref{rem:mixture}, and the marginal $p_t(y) = \sum_{a\in\cA} p_0(a)p_t^{\mathrm{ac}}(y\given a)$ averages this mixture across anchors.

Our strategy is to learn the time-reversed SJD by training a single object: a denoising classifier $P_\theta(a\given y, t)$ fit by cross-entropy on SJD-corrupted pairs. We will show that an analytic reweighting of $P_\theta$ recovers both the score and the reverse hazard exactly --- so that \emph{no separate hazard head is needed} --- and that the whole training pipeline is simulation-free under standard conditions on the forward dynamics.

\subsection{When and where to stick}\label{sec:factor}
By \Cref{thm:anchored-time-reversal}, the sticky jump law has per-anchor backward intensities
\begin{equation}
    \Lambda^{\star}_t(a\given y) \equiv \backvec J_t(y,\{a\})= \frac{m_t(a)}{p_t(y)} \vec\lambda_t r_t(y\given a),\quad y\in\mathsf{X}_\cA,\ a\in\cA.
\label{eq:true-per-anchor-intensity}
\end{equation}
Summing over anchors gives the hazard at which $y$ commits, regardless of the destination,
\begin{equation}
\backvec{\lambda_t^\star}(y):=\sum_{a\in\cA}\Lambda^\star_t(a\given y)=\frac{\vec\lambda_t \sum_{a\in\cA}m_t(a) r_t(y\given a)}{p_t(y)},
\label{eq:true-total-intensity}
\end{equation}
and normalizing gives the allocation probability over destinations, conditional on a sticky jump,
\begin{equation}
\pi^\star_t(a\given y) := \frac{\Lambda^\star_t(a\given y)}{\sum_{b\in\cA}\Lambda^\star_t(b\given y)} = \frac{m_t(a)r_t(y\given a)}{\sum_{b\in\cA} m_t(b) r_t(y\given b)}.
\label{eq:stick-posterior}
\end{equation}
These are the reverse sampler's two decisions: \emph{when} and \emph{where} to commit, and together they fully factorize the sticky jump law:
\begin{equation}\label{eq:factorization}
\Lambda_t^\star(a\given y) = \backvec{\lambda_t^\star}(y)\pi_t^\star(a\given y), \qquad \sum_{a\in\cA} \pi_t^\star(a\given y) = 1.
\end{equation}
Crucially, $p_t(y)$ cancels from the allocation \eqref{eq:stick-posterior}; so does the forward hazard whenever it is \emph{anchor-agnostic}, i.e. $\vec\lambda_t(a) \equiv \vec\lambda_t$.
The reverse hazard \eqref{eq:true-total-intensity} thus carries all of the dependence on the unknown $p_t(y)$ that we still need to eliminate.

\subsection{Denoising posterior and classifier training}\label{sec:classifier}
Define the SJD denoising posterior
\begin{equation}\label{eq:denoising-posterior}
p_t(a\given y) := \Pr\big(\vec X_0 = a \bigm| \vec X_t = y\big) = \frac{p_0(a)p_t^{\mathrm{ac}}(y\given a)}{p_t(y)}, \quad a\in\cA,\ y\in\mathsf X_\cA,
\end{equation}
where $p_t^{\mathrm{ac}}(\cdot\given a)$ is the mixture from \Cref{rem:mixture}. With paired samples $(\vec X_0, \vec X_t)$ drawn $\vec X_0 \sim p_0$ and $\vec X_t \sim p_t^{\mathrm{ac}}(\cdot\given \vec X_0)$, we train a classifier $P_\theta(a\given y, t)$ by the standard cross-entropy
\begin{equation}\label{eq:classifier-loss}
\mathcal L(\theta) := \E\left[-\log P_\theta\big(\vec X_0 \bigm| \vec X_t, t\big)\right],
\end{equation}
whose Bayes-optimal predictor is exactly \eqref{eq:denoising-posterior}.

The classifier sees $y$ and $t$ but not the latent unstick time $\tau$ that produced $\vec X_t$ from $\vec X_0$ which is the same information pattern available to the reverse sampler at inference. This is by design: the time-reversal theorem operates on $p_t$, not on trajectories, and has already integrated out $\tau$ by flux balance.

\subsection{Denoising hazard matching (DHM)}\label{sec:dhm}
The reverse hazard $\backvec{\lambda_t^\star}(y)$ \eqref{eq:true-total-intensity} still contains the unknown marginal $p_t(y)$. We eliminate it by writing $\backvec{\lambda_t^\star}$ as a conditional expectation of a forward-known quantity.

\begin{definition}[DHM target]\label{def:dhm-target}
For $\vec X_0 \in \cA$  and $y \in \mathsf X_\cA$, define
\begin{equation}\label{eq:dhm-target}
\widehat\lambda_t(y, \vec X_0) := \frac{\vec\lambda_t(\vec X_0) S_{\vec X_0}(t) r_t(y\given \vec X_0)}{p_t^{\mathrm{ac}}(y\given \vec X_0)},
\end{equation}
where $S_{\vec X_0}(t) = \exp\big(-\int_0^t \vec\lambda_s(\vec X_0)\dd s\big)$ is the survival.
\end{definition}

The target involves only forward-specified quantities: the forward hazard, its survival, the unsticking density at the current time, and the off-anchor conditional from \Cref{rem:mixture}. It contains no intractable marginal $p_t(y)$.

\begin{theorem}[Denoising hazard matching]\label{thm:dhm}
Let $(\vec X_0, \vec X_t)$ be paired by the SJD forward law, $\vec X_0 \sim p_0$ and $\vec X_t \sim p_t^{\mathrm{ac}}(\cdot\given \vec X_0)$. Then
\begin{equation}\label{eq:dhm-identity}
\E\left[\widehat\lambda_t(\vec X_t, \vec X_0) \bigm| \vec X_t = y\right] = \backvec{\lambda_t^\star}(y)
\qquad \text{for almost every } y \in \mathsf X_\cA.
\end{equation}
\end{theorem}

\Cref{thm:dhm} is the hazard analogue of denoising score matching \citep{vincent2011denoisingscorematching}: the reverse hazard, like the score, is an intractable functional of the marginal that nonetheless equals a tractable conditional expectation under the same forward corruption used for training.

\subsection{Plug-in: one classifier, no hazard head}\label{sec:plugin}
Because $\vec X_0$ is discrete, \eqref{eq:dhm-identity} expands as a finite sum,
\begin{equation}\label{eq:rb-identity}
\backvec{\lambda_t^\star}(y) = \sum_{a\in\cA} \widehat\lambda_t(y, a) p_t(a\given y),
\end{equation}
which substitutes a known weighting against an unknown distribution. Replacing $p_t(\cdot\given y)$ with the trained classifier $P_\theta$ yields the Rao--Blackwellized plug-ins
\begin{equation}\label{eq:plugin-estimators}
\Lambda_{\mathrm{plug},\theta}(a\given y, t) := \widehat\lambda_t(y, a) P_\theta(a\given y, t), \qquad \lambda_{\mathrm{plug},\theta}(y, t) := \sum_{a\in\cA} \Lambda_{\mathrm{plug},\theta}(a\given y, t).
\end{equation}
Both ingredients of the factorization \eqref{eq:factorization} are now in hand: the total commit rate is $\lambda_{\mathrm{plug},\theta}(y, t)$, and the anchor allocation at a jump is
\begin{equation}\label{eq:plugin-allocation}
\pi_{\mathrm{plug},\theta}(a\given y, t) = \frac{\Lambda_{\mathrm{plug},\theta}(a\given y, t)}{\lambda_{\mathrm{plug},\theta}(y, t)} \propto \widehat\lambda_t(y, a) P_\theta(a\given y, t).
\end{equation}
If $P_\theta(\cdot\given y, t) = \mathrm{softmax}(\ell_\theta(\cdot; y, t))$, sampling the destination anchor reduces to a softmax over reweighted logits $\ell_\theta(a; y, t) + \log \widehat\lambda_t(y, a)$. Training is plain cross-entropy; all SJD-specific structure enters at inference through the analytic reweighting $\widehat\lambda_t$. 
\Cref{app:rb-plugin} describes three properties that justify \eqref{eq:plugin-estimators}.

\subsection{Simulation-free training}\label{sec:sim-free}
Sampling $\vec X_t \sim p_t^{\mathrm{ac}}(\cdot\given \vec X_0)$ is simulation-free under three conditions on the forward process.

\begin{enumerate}
    \item \emph{Closed-form survival.} The survival $S_a(\tau)$ should be available in closed form so that an unstick time $\tau$ can be drawn by inverse CDF on $(0, t)$. For hazards of the form $\vec\lambda_\tau(a) = \beta(\tau) w(a)$ with $\beta$ elementary (e.g. constant, linear, polynomial, piecewise-constant), $S_a$ is a single exponential.
    \item \emph{Gaussian closure of the convolution.} The convolution $r_\tau(\cdot\given a) * K_{\tau\to t}$ should be Gaussian in closed form. This holds whenever $r_\tau$ is Gaussian and the forward SDE is linear-Gaussian.
    \item \emph{One-sample unbiasedness.} A single draw $(\tau, \vec X_t)$ is an unbiased sample from $p_t^{\mathrm{ac}}(\cdot\given \vec X_0)$. The cross-entropy loss therefore receives an unbiased gradient from one sample per step; we never need to evaluate the mixture density $p_t^{\mathrm{ac}}(y\given a)$ during training.
\end{enumerate}

\begin{examplecont}\label{ex:vp-matched-training}
\Cref{ex:vp-matched} satisfies all three conditions. The VP-SDE is linear-Gaussian; for an anchor-agnostic hazard $\vec\lambda_t(\vec X_0) \equiv \vec\lambda(t)$, the survival $S(t) = \exp\big(-\int_0^t \vec\lambda(s)\dd s\big)$ is a single exponential and $\tau$ is sampled by inverse CDF in closed form. The convolution evaluates to a single Gaussian whose mean is the blended mean and whose variance $v_t(\tau)$ (\Cref{app:convolution}):
\begin{equation}\label{eq:vp-matched-convolution}
(r_\tau * K_{\tau\to t})(\cdot\given \vec X_0, i) = \cN\Big(\cdot;\ \alpha(t)\mu_i(\vec X_0),\ v_t(\tau)I_d\Big),
\quad
v_t(\tau)= \sigma^2(t) - (1-\eta^2)\tfrac{\alpha^2(t)}{\alpha^2(\tau)}\sigma^2(\tau).
\end{equation}
Sampling $\vec X_t$ given $(\vec X_0, \tau)$ is a single Gaussian draw at each position, with the per-position mean $\alpha(t)\mu_i(\vec X_0)$ obtained from one $L\times L$ matmul $\mu = W E(\vec X_0)$.
\end{examplecont}

\Cref{alg:sjd-training} ties the pieces together. In practice, it is vectorized across a minibatch: the never-unstuck branch is implemented as a loss mask. \Cref{app:sampling} and \Cref{alg:sjd-sampling} describe how to sample from a trained SJD. When $W\neq I$, $W$ enters the sampler only through the reverse drift: at $\eta = 1$ the classifier-induced score evaluates the residual to the blended posterior mean at the cost of one blending product per reverse step, while the commit rates and allocation are unchanged (\Cref{app:classifier-score}).

\paragraph{Cost.} Despite the added machinery, SJD trains and samples at the same cost as existing hybrids. Training is one classifier forward--backward pass per step on a cross-entropy loss, the same as a masked or hybrid baseline, because all SJD-specific structure enters through the jump corruption: closed-form draws of $(\tau, \vec X_t)$ and one $L\times L$ blending product per minibatch. At sampling, each reverse step is one forward pass through $P_\theta$ plus the analytic reweighting, an $N_\tau$-point quadrature at $O(N_\tau|\cA|d)$ that runs one to two orders of magnitude below the network pass (\Cref{app:plugin}).

\begin{figure}[t]
\centering
\begin{minipage}[t]{0.48\textwidth}
\begin{algorithm}[H]
\caption{SJD simulation-free training.}\label{alg:sjd-training}
\begin{algorithmic}[1]
\REQUIRE Prior $p_0$; $\vec\lambda_t$, $S_a$, $r_\tau$, $K_{\tau\to t}$; embedding $E$ and blending matrix $W$; $P_\theta$; $T$.
\WHILE{not converged}
\STATE Sample $\vec X_0 \sim p_0$, $t \sim \mathrm{Unif}(0, T)$.
\STATE Compute $\mu \leftarrow W E(\vec X_0)\in\R^{L\times d}$.
\STATE Sample $u \sim \mathrm{Unif}(0, 1)$.
\STATE \textbf{if} $u < S_{\vec X_0}(t)$ \textbf{then continue}
\STATE Sample $\tau\!\sim\!\vec\lambda_\tau(\vec X_0) S_{\vec X_0}(\tau)/(1\!-\!S_{\vec X_0}(t))$.
\STATE Sample $\vec X_t \sim (r_\tau * K_{\tau\to t})(\cdot\given \vec X_0, i)$ at each unstuck position $i$, using \eqref{eq:vp-matched-convolution}.
\STATE Step $\theta$ on $-\log P_\theta(\vec X_0 \given \vec X_t, t)$.
\ENDWHILE
\end{algorithmic}
\end{algorithm}
\end{minipage}\hfill
\begin{minipage}[t]{0.48\textwidth}
\begin{algorithm}[H]
\caption{SJD predictor sampling.}\label{alg:sjd-sampling}
\begin{algorithmic}[1]
\REQUIRE $P_\theta$; $\vec b_t$, $g_t$, $\vec\lambda_t$, $r_t$; $T$, $\Delta\tau$.
\STATE $y \sim p_T$; $\tau \leftarrow 0$; mark all sites active.
\WHILE{$\tau < T$ and any site is active}
\STATE $t \leftarrow T - \tau$.
\STATE One forward pass: $P_\theta(\cdot\given y, t)$, $s_\theta(y, t)$.
\STATE Evaluate $\widehat\lambda_t(y, a)$ for $a \in \cA$ by \eqref{eq:plugin-quadrature}.
\STATE Form $\lambda_{\mathrm{plug},\theta}, \pi_{\mathrm{plug},\theta}$ by \eqref{eq:plugin-estimators}.
\STATE \emph{Drift:} $y\leftarrow y + (-\vec b_t(y) + g_t^2 s_\theta)\Delta\tau + g_t\sqrt{\Delta\tau}\xi$.
\STATE \emph{Jump:} per active site, with prob.\ 
       $1 - e^{-\lambda_{\mathrm{plug},\theta}\Delta\tau}$, draw 
       $a \sim \pi_{\mathrm{plug},\theta}$, freeze site at $a$.
\STATE $\tau \leftarrow \tau + \Delta\tau$.
\ENDWHILE
\STATE \textbf{return } $y$.
\end{algorithmic}
\end{algorithm}
\end{minipage}
\end{figure}

\section{Connections to existing diffusions}\label{sec:connections}
The SJD construction of \Cref{sec:sjd} unifies several existing model classes by varying three forward-process levers: the unsticking kernel $r_a$, the forward hazard $\vec\lambda_t$, and, in the VP-matched family of \Cref{ex:vp-matched}, the blending matrix $W$ and off-anchor variance ratio $\eta$. Different limits of these quantities recover absorbing-state discrete diffusion, pure continuous diffusion, and recently-proposed hybrid diffusion. We make these correspondences precise below.

\subsection{Masked diffusion as a degenerate-kernel limit}\label{sec:conn-discrete}
Masked diffusion models (MDMs) \citep{austin2021d3pm,sahoo2024mdlm,shi2024md4} are the SJD whose off-anchor region collapses to a single absorbing point. Identify $\cA = \{e(v) : v \in \cV\}$ with the embedded vocabulary, fix a mask location $m \notin \cA$, and set
\begin{equation}\label{eq:mdlm-as-sjd} 
    r_a(\cdot, t) \equiv \delta_m, \qquad g_t \equiv 0, \qquad
\vec\lambda_t(a) = \beta(t) := -\dot\alpha_t/\alpha_t.
\end{equation}
Then $p_t^{\mathrm{ac}}(\cdot\given a) = (1-\alpha_t)\delta_m$, and the plug-in \eqref{eq:plugin-allocation} at $y = m$ reduces to $\pi_{\mathrm{plug},\theta}(a\given m, t) \propto P_\theta(a\given m, t)$ with commit rate $\beta(t)\alpha_t/(1-\alpha_t)$, the standard MDM step. SJD strictly contains this limit and additionally provides an absolutely continuous off-anchor region whenever $r_a$ has a Lebesgue density.

This limit makes the ``information void'' of \citet{zheng2025cadd} structural: because every anchor unsticks to the same point, the conditional law $p^{\mathrm{ac}}_t(\cdot|a) = (1-\alpha_t)\delta_m$ is anchor-independent, so the corrupted state at a masked position carries no evidence about its source token. Correspondingly, the spatial factor of the DHM target~\eqref{eq:dhm-target} reduces to the constant $1/(1-\alpha_t)$: the commit rate is time-only, and the allocation can be informed only from context, through the classifier (\cref{sec:conn-schedules}).

\subsection{Continuous diffusion as the no-jump limit}\label{sec:conn-continuous}
Continuous diffusion is the limit in which every anchor releases its mass at $t = 0^{+}$, that is $S_a(0^{+}) = 0$ for all $a \in \mathcal A$, after which $\vec\lambda_t \equiv 0$. For $t>0$ the marginal $p_t$ then has no atomic part, \eqref{eq:forwardSDE} is a standard SDE, and Theorem~\ref{thm:anchored-time-reversal} reduces to the classical score-based time-reversal of \citet{song2021sde}. Because $p_t(\{a\}) = 0$ for every anchor, the flux identity~\eqref{eq:flux} assigns zero reverse intensity to every $y \in \mathsf{X}_{\mathcal A}$: the absence of atoms in the forward marginal is precisely what removes the reverse jumps.
For discrete data, three external mechanisms appear in the literature. \emph{(i) Decoder:} train a continuous diffusion in latent space and push samples through a separately-trained decoder \citep{li2022diffusionlm,dieleman2022cdcd,lovelace2023ld4lg}. \emph{(ii) Quantization:} Bit Diffusion \citep{chen2023analogbits} thresholds the terminal state. \emph{(iii) Doob's $h$-transform:} \citet{liu2023constraindiff} bias the SDE by the gradient of a singular log-density so its terminal law lands on a discrete set. A recent wave of embedding- and one-hot-space models closes the gap to discrete diffusion \citep{yang2026replaid,hu2026elf,chen2026langflow,lee2026flm}; in all of them, discreteness is still restored by a step outside the dynamics, whether a decoder, a readout, or an argmax. SJD instead enforces discreteness through the dynamics: anchors are absorbing in reverse time as a direct consequence of \eqref{eq:flux}.

\subsection{Hybrid diffusion}\label{sec:conn-hybrid}
Hybrid diffusion models such as CADD \citep{zheng2025cadd}, CANDI \citep{pynadath2025candi}, and CCDD \citep{zhou2025ccdd} couple a masked chain with a continuous latent that carries graded information at masked positions. The VP-matched SJD of \cref{ex:vp-matched} at $(W, \eta) = (I, 1)$ recovers this construction exactly. With $W = I$, the unsticking kernel reduces to the per-anchor form $r_\tau(\cdot\given a)$, with $\eta = 1$ it coincides with the VP perturbation kernel ($r_\tau = q_\tau$), and the semigroup property collapses \Cref{rem:mixture} to
\begin{equation}\label{eq:eta1-collapse}
    p_t^{\mathrm{ac}}(y\given a) = (1-S(t)) \mathcal N\big(y;\alpha(t)a,\sigma^2(t)I\big).
\end{equation}

\begin{proposition}[The plain hybrid coincides with CADD]\label{prop:plain-hybrid}
Consider the VP-matched SJD of \Cref{ex:vp-matched} at $(W,\eta)=(I,1)$ with anchor-agnostic hazard $\vec\lambda(t)$ and survival $S(t)$, and CADD with masking survival $\alpha_t$, continuous kernel $\cN(\sqrt{\bar\gamma_t}E(x_0),(1-\bar\gamma_t)I_d)$ at masked positions, and a common embedding map $E$. Identify $S(t)\leftrightarrow\alpha_t$ and $(\alpha(t),\sigma^2(t))\leftrightarrow(\sqrt{\bar\gamma_t},1-\bar\gamma_t)$. Then:
\begin{enumerate}[label=(\roman*),leftmargin=*,nosep]
    \item the time-$t$ laws of the corrupted state given the data coincide for every $t$, and the SJD cross-entropy objective of \Cref{alg:sjd-training} equals CADD's cross-entropy objective;
    \item the reverse process of \Cref{thm:anchored-time-reversal} reduces to a commit rate $\vec\lambda(t)S(t)/(1-S(t))$ shared across anchors, an allocation equal to the classifier posterior, and a reverse SDE whose score is the residual to the posterior-mean embedding;
    \item CADD's sampler is a first-order discretization of (ii), so the two generative laws coincide in the small-step limit.
\end{enumerate}
\end{proposition}

The sampler that hybrids assemble component by component thus arise jointly as the time reversal of one forward process, with rate, destination, and drift fixed by flux balance \eqref{eq:flux}. This turns the unsticking kernel into a new design space: the continuous coordinate is reversed state rather than conditioning information, so any forward kernel yields exact reverse dynamics through \Cref{thm:anchored-time-reversal}.

\subsection{Sticky boundaries: a classical antecedent}\label{sec:conn-sticky}
SJD's theoretical antecedent is the classical theory of sticky-boundary diffusions \citep{feller1952parabolic,feller1954diffusion,ito1963brownianhalfline}. A one-dimensional sticky diffusion holds at the boundary for an exponential time and re-enters the interior according to a fixed jump-out measure; its law splits as an interior density plus a boundary atom \citep{bourabee2020sticky}. SJD lifts this model to a multi-anchor, high-dimensional setting: $\cA$ replaces the boundary point, $\vec\lambda_t$ replaces the sticky coefficient, and $R_a$ replaces the jump-out measure, with \Cref{rem:mass-split} as the corresponding mass split.

\subsection{Commitment schedules and token ordering}\label{sec:conn-schedules}
The DHM target \eqref{eq:dhm-target} factors into temporal and spatial components:
\[
\widehat\lambda_t(y, \vec X_0) = \underbrace{\vec\lambda_t(\vec X_0) S_{\vec X_0}(t)}_{\text{temporal}} \cdot \underbrace{\frac{r_t(y\given \vec X_0)}{p_t^{\mathrm{ac}}(y\given \vec X_0)}}_{\text{spatial}}.
\]
At $(W, \eta) = (I, 1)$, \eqref{eq:eta1-collapse} forces the spatial factor to $1/(1-S(t))$ and the allocation \eqref{eq:plugin-allocation} reduces to the classifier posterior, recovering the time-only schedule of MDM-style samplers. Choosing $\eta < 1$ introduces a per-anchor mixture over unstick times whose density depends on proximity $\|y - \alpha(t)E(a^{(i)})\|^2$; $W \neq I$ then replaces the per-anchor distance with a cross-position distance to the blended mean $\mu_i(\vec X_0)$, so the hazard ratio at one position depends on the configuration of its neighbors. This gives a process-level reading of adaptive ordering \citep{kim2025trainforworst,chao2025mdmprime}: at $(I,1)$ the state carries no timing signal, so a planner can only rank positions by classifier confidence, an ordering imposed at sampling time. State-dependent forward hazards $\vec\lambda_t(a)$, in the spirit of MD4's GenMD4 \citep{shi2024md4}, are a further axis (See \cref{app:e2e-hazard}).  We fix $\eta=1$ in all experiments and recommend it as the practical default: the kernel then coincides with the perturbation kernel, so the identity-blend member is exactly CADD (\Cref{prop:plain-hybrid}), and the commit rate and allocation are closed-form and identical for every $W$ (\Cref{app:classifier-score}).

\section{Experiments}\label{sec:experiments}
The cross-position structure in \Cref{ex:vp-matched,sec:conn-hybrid} should be exploitable whenever positions are statistically dependent under a known incidence structure, so that a fixed linear functional of a position's neighbors is itself diagnostic of that position's value. 
Such structure is common in real-world data, and we instantiate three examples: spatial locality on CIFAR-10 through a value-space blur, local sequential correlation on Text8, and the row/column/box constraint graph on Sudoku. 

\subsection{Image generation}\label{sec:exp-image}
\begin{wrapfigure}[18]{r}{0.5\linewidth}
  \centering
  \vspace{-\intextsep}
  \includegraphics[width=\linewidth]{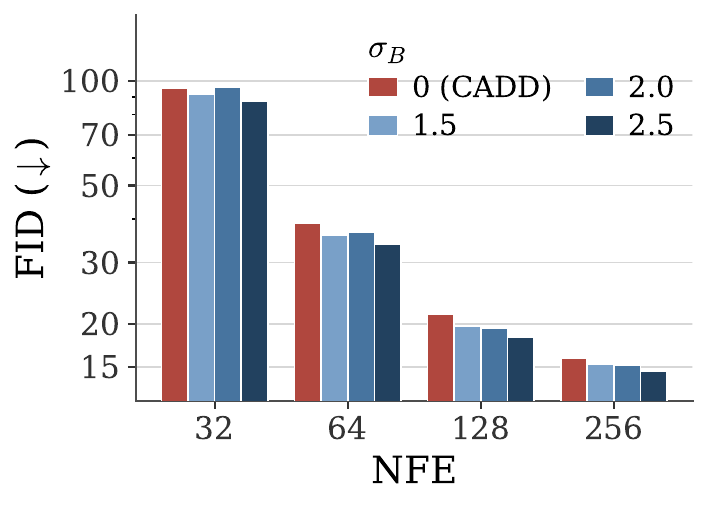}
  \caption{CIFAR-10 FID ($\downarrow$) across sampling budgets as a function of the value-space blur bandwidth $\sigma_B$. $\sigma_B{=}0$  is the plain hybrid, which coincides with CADD \citep{zheng2025cadd} (\Cref{prop:plain-hybrid}); $\sigma_B\in\{1.5,2.0,2.5\}$ are SJD.}
  \label{fig:cifar10_sigmaB}
\end{wrapfigure}
\Cref{fig:cifar10_sigmaB} sweeps $\sigma_B\in\{1.5,2.0,2.5\}$, the spatial correlation length of the corruption, against the $\sigma_B=0$ endpoint across sampling budgets. The identity kernel, the point at which existing hybrids sit, is the weakest member of the sweep: the corruption geometry is a substantive degree of freedom on images as well. Full settings are in \Cref{app:cifar10-imagenet}. At the full budget of 256 NFE, \Cref{tab:cifar10_fid} compares the best bandwidth from the sweep, $\sigma_B{=}2.5$, with masked, continuous, and hybrid baselines under matched settings: SJD attains the best FID overall, 1.3 points below the identity-kernel hybrid, with DDPM the closest external baseline.

\subsection{Text generation}\label{sec:exp-text8}
Setting $W$ to a 1D Gaussian convolution of bandwidth $\sigma_W$ replaces a token with a noisy blend of a local neighborhood. This is a natural inductive bias for characters, where adjacent characters are correlated, so a blended window ties each position's corruption to its local context.
\Cref{fig:text8} reports valid-word counts at length thresholds $\geq 5$ and $\geq 6$ for NFE budgets of 32, 64, and 128. For $\sigma_W \geq 1.0$ SJD surpasses both CANDI \citep{pynadath2025candi} and MDLM \citep{sahoo2024mdlm}, with a larger margin at the $\geq 6$ threshold than at $\geq 5$ and gains that hold at NFE $= 128$. \Cref{fig:text8-frontier} further emphasizes this finding with quality–diversity frontiers on Text8, by NFE. Full settings are in \Cref{app:text8}.
For length $\geq 5$, SJD at 32 NFE exceeds both terminal-rounding baselines at 128 NFE (\Cref{app:terminal-decoding}); a likelihood-bound comparison with MD4 appears in \Cref{app:md4-likelihood}.

\begin{figure}
    \centering
    \includegraphics[width=\linewidth]{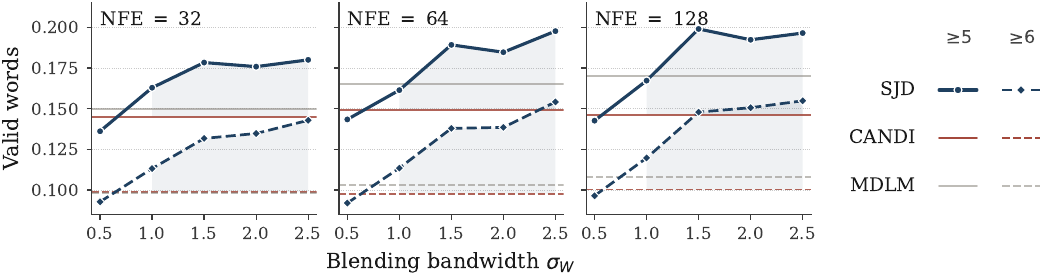}
    \caption{Text8 valid-word counts at length thresholds $\geq 5$ and $\geq 6$, as a function of the blending-kernel bandwidth $\sigma_W$, across NFE budgets. CANDI corresponds to $\sigma_W = 0$ ($W = I$); MDLM is the degenerate-kernel limit.}
    \label{fig:text8}
\end{figure}

\subsection{Solving sudokus}\label{sec:exp-sudoku}
Sudoku probes the blending axis $W$ when the governing structure is a hard combinatorial constraint graph. The task is to complete a $9\times9$ grid so that every row, column, and $3\times3$ box is a permutation of $\{1,\dots,9\}$, which requires the model to honor board-level constraints rather than local statistics. We encode those constraints directly into the unsticking kernel. The full construction is given in \Cref{app:sudoku}. We then ask whether this constraint-aware corruption helps the model learn the board's logic. \Cref{tab:sudoku_schedule_ablations} isolates three effects of $W$. Across $6$ seeds CADD is bimodal: several runs collapse to near-chance performance. SJD never collapses. Moreover, by board takeoff (the training step at which a seed first resolves complete boards) SJD takes off roughly $4\times$ sooner; even the CADD seeds that avoid collapse begin solving boards far later. Finally, SJD's mean board accuracy exceeds CADD's and the autoregressive solver granted the oracle solving order.

\begin{table}[t]
\centering
\setlength{\tabcolsep}{3.5pt}
\small
\begin{minipage}{0.48\linewidth}
\centering
\captionof{table}{Sudoku board accuracy and convergence. \emph{Acc.}\ is exact-match board accuracy. \emph{Takeoff} is the average training step at which seeds first begin to solve complete boards.}
\label{tab:sudoku_schedule_ablations}
\begin{tabular}{@{}lcc@{}}
\toprule
Method & Acc.\ (\%) $\uparrow$ & Takeoff $\downarrow$ \\
\midrule
\multicolumn{3}{@{}l}{\textit{Autoregressive \citep{shah2024puzzles}}}\\
\quad Left-to-right & 9.73 & NA \\
\quad Solver order  & 87.18 & NA \\
\midrule
\multicolumn{3}{@{}l}{\textit{Masked Diffusion \citep{kim2025trainforworst}}}\\
\quad MDM & 89.49 & NA\\
\midrule
\multicolumn{3}{@{}l}{\textit{Hybrid Diffusion}}\\
\quad CADD [$(W{,}\eta){=}(I{,}1)$] & 47.12& 203k \\
\quad SJD ($\sigma_W{=}1.5$) & \textbf{95.65} & \textbf{50k}\\ 
\bottomrule
\end{tabular}
\end{minipage}
\hfill
\begin{minipage}{0.48\linewidth}
\centering
\captionof{table}{FID on CIFAR-10 32$\times$32 for unconditional generation; all methods sampled at 256 NFE.}
\label{tab:cifar10_fid}
\begin{tabular}{@{}lc@{}}
\toprule
Method & FID $\downarrow$ \\
\midrule
\multicolumn{2}{@{}l}{\textit{Hybrid Diffusion}} \\
\quad CADD [$(W{,}\eta){=}(I{,}1)$]          & 15.88 \\
\quad SJD ($\sigma_B{=}2.5$)               & \textbf{14.57} \\
\midrule
\multicolumn{2}{@{}l}{\textit{Masked Diffusion}} \\
\quad MD4 \citep{shi2024md4}               & 17.43 \\
\quad MDLM \citep{sahoo2024mdlm}           & 18.11 \\
\midrule
\multicolumn{2}{@{}l}{\textit{Continuous Diffusion}} \\
\quad Bit Diff. \citep{chen2023analogbits} & 19.82 \\
\quad DDPM \citep{ho2020ddpm}              & \underline{14.66} \\
\bottomrule
\end{tabular}
\end{minipage}
\end{table}

\section{Conclusion}\label{sec:conclusion}
sticky jump diffusions unify masked, continuous, and hybrid diffusion through a single forward process. Flux balance determines the reverse commit rates and destinations, which denoising hazard matching estimates from one classifier trained by simulation-free cross-entropy. Cross-position unsticking kernels encode spatial, sequential, or constraint dependencies and improve generation over the identity-kernel hybrid. Our experiments use VP-matched kernels at on CIFAR-10, Text8, and Sudoku; evaluation at language-model scale and with other kernel families remains open.

\clearpage

\begin{ack}
This work was supported by the National Science Foundation under grant
IIS-2508145 and by an Amazon Research Award.
\end{ack}

%%%%%%%%%%%%%%%%%%%%%%%%%%%%%%%%%%%%%%%%%%%%%%%%%%%%%%%%%%%%

\bibliography{references}
\bibliographystyle{plainnat}
\clearpage
%%%%%%%%%%%%%%%%%%%%%%%%%%%%%%%%%%%%%%%%%%%%%%%%%%%%%%%%%%%%

\appendix

\section{Notation, assumptions, and vocabulary}\label{app:notation}
We use overhead arrows to disambiguate time direction:
\begin{itemize}[leftmargin=*,nosep]
\item $\vec{\cdot}$ denotes \emph{forward} (noising) objects, e.g., $\vec X_t$ (state), $\vec b_t$ (drift), $\vec J_t$ (jump kernel), and $\vec\lambda_t$ (forward hazard of leaving an anchor).
\item $\backvec{\cdot}$ denotes \emph{reverse-time} (denoising) objects, e.g., $\backvec X_\tau$, $\backvec b_t$, $\backvec J_t$, with reverse time $\tau = T-t$.
\end{itemize}

The state space is $\mathsf X=\R^d$ and the (finite or countable) set of anchors is $\cA=\{a_k\}_{k=1}^K\subset \R^d$. 
We write the continuous region as
\[
\mathsf X_\cA := \R^d \setminus \cA .
\]

The forward process $(\vec X_t)_{t\in[0,T]}$ is an Itô diffusion on $\mathsf X_\cA$ with scalar diffusion coefficient $g_t>0$ and drift $\vec b_t$, and it can jump off anchors:
\[
\dd \vec X_t = \vec b_t(\vec X_t)\dd t + g_t\dd \vec W_t \quad \text{if } \vec X_t\in \mathsf X_\cA,
\]
and, if $\vec X_t=a\in\cA$, then with hazard $\vec\lambda_t$ it unsticks to $Y\sim R_a(\cdot)$ where $R_a$ is a fixed probability kernel supported on $\mathsf X_\cA$. 
There are no forward jumps into anchors: if $\vec X_t\in\mathsf X_\cA$ the jump intensity into $\cA$ is zero.

The reverse-time process $(\backvec X_\tau)_{\tau\in[0,T]}$ with $\tau=T-t$ is Markov on $\R^d$ and (under Assumption~\ref{ass:reg}) satisfies on $\mathsf X_\cA$:
\[
\dd \backvec X_\tau = \backvec b_t(\backvec X_\tau)\dd \tau + g_t\dd \backvec W_\tau, \qquad  \backvec b_t(x) = \vec b_t(x) - g_t^2 \nabla \log p_t(x).
\]

We write $R_a(\dd y)$ for the forward unsticking kernel (a probability measure on $\mathsf X_{\cA}$).
Whenever $R_a$ admits a density w.r.t.\ Lebesgue on $\mathsf X_{\cA}$, we write
\(
R_a(\dd y) = r_a(y) \dd y.
\)
Likewise, whenever $p_t$ admits a density on $\mathsf X_{\cA}$ we write $p_t(\dd y)=p_t(y)\dd y$.
All pointwise formulas involving $p_t(y)$ and $r_a(y, t)$ are understood under these absolute-continuity assumptions.

For sequence tasks, $E:\cA\to\R^d$ denotes the embedding map and $W\in\R^{L\times L}$ a blending matrix on positions. The blended mean at position $i$ is $\mu_i(\vec X_0) := \big(WE(\vec X_0)\big)_i$; we recover the per-anchor case $\mu_i = E(a^{(i)})$ at $W = I$.

\section{Preliminaries: existing diffusion frameworks}\label{app:prelim}
We review the model classes that \Cref{sec:connections} identifies with limits of the SJD construction of \Cref{sec:sjd}.

\subsection{Continuous score-based diffusion}\label{app:prelim-continuous}
A score-based diffusion model on $\R^d$ \citep{sohl2015thermo,ho2020ddpm,song2021sde} runs the Itô SDE $\dd \vec X_t = \vec b_t(\vec X_t)\dd t + g_t \dd \vec W_t$ from $\vec X_0 \sim p_{\mathrm{data}}$ to a tractable prior $p_T$. We focus on the variance-preserving (VP) instance with $\vec b_t(x) = -\tfrac12\beta(t)x$ and $g_t = \sqrt{\beta(t)}$, whose conditional marginal kernel is Gaussian,
\[
q_t(\cdot\given x_0) = \cN\big(\cdot; \alpha(t) x_0, \sigma^2(t) I_d\big), \qquad \alpha(t) = \exp\Big(-\tfrac12 \int_0^t\beta(s)\dd s\Big),\ \ \sigma^2(t) = 1 - \alpha^2(t),
\]
and converges to $\cN(0, I_d)$ as $t \to T$. The reverse-time process is the Itô SDE $\dd \backvec X_\tau = (-\vec b_t(\backvec X_\tau) + g_t^2 \nabla \log p_t(\backvec X_\tau))\dd \tau + g_t \dd \backvec W_\tau$ with $\tau = T-t$ \citep{nelson1967dynamical,Anderson1982ReversetimeDE,song2021sde}. The score $\nabla \log p_t$ is learned by denoising score matching \citep{hyvarinen2005scorematching,vincent2011denoisingscorematching}: a network $s_\theta(y,t)$ minimizes $\E[\|s_\theta(y,t) - \nabla_y \log q_t(y\given x_0)\|^2]$, the conditional score being available in closed form for VP. For discrete data, this construction requires an external map to the support like a separately-trained decoder \citep{li2022diffusionlm,dieleman2022cdcd,lovelace2023ld4lg}, a quantization step \citep{chen2023analogbits}, or a Doob $h$-transform \citep{liu2023constraindiff}.

\subsection{Masked diffusion}\label{app:prelim-masked}
Absorbing-state masked diffusion \citep{austin2021d3pm,sahoo2024mdlm,shi2024md4} operates on a discrete state space $\cV \cup \{\mathsf m\}$ with a special mask symbol $\mathsf m$. The forward process keeps each clean token with probability $\alpha_t$ and replaces it by $\mathsf m$ otherwise, where $\alpha_t$ is a smooth survival schedule with $\alpha_0 = 1$ and $\alpha_T = 0$; in continuous time the clean-to-mask hazard is $\beta(t) = -\dot\alpha_t/\alpha_t$. Time reversal yields a Markov chain in which masked positions un-mask at total rate $\beta(t)\alpha_t/(1-\alpha_t)$ to a token drawn from the denoising posterior $p_t(v \given \mathsf m, x_t) \propto p_0(v)\Pr(\mathsf m \given v, t)$, while clean positions remain absorbing. A classifier $P_\theta(v \given x_t, t)$ is trained by cross-entropy on masked positions and substitutes for the posterior at sampling time. The per-position factorization is exact in the forward law and couples positions only through the classifier's context input; once a position commits, the only signal carried about its proximity to alternative tokens is whatever the classifier reconstructs from the surviving context.

\subsection{Hybrid discrete-continuous diffusion}\label{app:prelim-hybrid}
Hybrid models \citep{zheng2025cadd,pynadath2025candi,zhou2025ccdd} couple a masked chain with a continuous latent that carries graded information at masked positions. CADD \citep{zheng2025cadd} represents the state at time $t$ as a pair $(x_t, z_t) \in (\cV \cup \{\mathsf m\}) \times \R^d$. The discrete component $x_t$ evolves as the masked chain of \Cref{app:prelim-masked} with survival $\alpha_t$, and the continuous component, defined for masked sites, follows the VP perturbation kernel of an embedding,
\[
z_t \given x_0 \sim \cN\big(\sqrt{\bar\gamma_t} E(x_0), (1-\bar\gamma_t) I_d\big), \qquad \text{when the site is masked at time $t$},
\]
with $\bar\gamma_t$ a VP-style schedule and $E$ a fixed embedding map. Time reversal couples the discrete un-masking transition with a score-driven update on $z_\tau$; both are predicted from a single network conditioned on $(x_t, z_t, t)$. CANDI \citep{pynadath2025candi} keeps the masked-position augmentation and varies the reverse parameterization. CCDD \citep{zhou2025ccdd} instead corrupts the continuous component at every position under a factored forward law $q_t(x_t,z_t\given x_0)=q_t^{\mathrm{disc}}(x_t\given x_0)q_t^{\mathrm{cont}}(z_t\given z_0)$, with possibly asynchronous schedules, so its forward lies outside the $(I,1)$ configuration.

\subsection{Framework-level unifications}\label{sec:conn-unification}
Generator Matching \citep{holderrieth2025generatormatching} learns the marginal generator of an arbitrary Markov process from conditional generators, and denoising Markov models \citep{benton2024dmm} extend score matching to general state spaces; both are abstract recipes that must be instantiated per process class. SJD is such an instantiation, worked out end to end for jump-diffusions whose marginals mix a continuous density with anchor atoms: the flux identity \eqref{eq:flux} gives the reverse jump kernel in closed form, DHM (\Cref{thm:dhm}) turns the reverse hazard into a conditional expectation that a classifier estimates, training is simulation-free, and the unsticking kernel exposes the $(W,\eta)$ design space our experiments study. Neither framework supplies these objects for the sticky, atom-carrying marginal.

\section{Theoretical details and proofs}\label{app:proofs}
\subsection{\texorpdfstring{Proof of \Cref{thm:anchored-time-reversal}}{Proof of Theorem \ref{thm:anchored-time-reversal}}}
\begin{assumption}[Regularity]\label{ass:reg}~
\begin{enumerate}
\item $\vec b_t$ and $g_t$ are continuous in $t$, locally Lipschitz in $x$, with $g_t\in(0,\infty)$.
\item For $x\in\cA$, $\vec J_t(x,\dd y)=\vec\lambda_t(x)R_t(\dd y\given x)$ on $\mathsf X_\cA$; for $x\in\mathsf X_\cA$, $\vec J_t(x,\cdot)\equiv 0$.
\item For almost every \ $t\in(0,T)$, $p_t$ admits a density on $\mathsf X_\cA$ and atoms on $\cA$.
\end{enumerate}
\end{assumption}
\begin{theorem*}[Time reversal for jump--diffusions, restated]
Under Assumption~\ref{ass:reg}, define the time-reversed process
\(
\backvec X_\tau := \vec X_{T-\tau},\ \tau\in[0,T],
\)
and denote the corresponding forward time by \(t:=T-\tau\).
Then \((\backvec X_\tau)_{\tau\in[0,T]}\) is Markov.

\medskip
\noindent
\textbf{(Diffusion part on }\(\mathsf X_\cA\)\textbf{).}
On the continuous region \(\mathsf X_\cA=\R^d\setminus\cA\), the reversed process admits an SDE
\[
\dd \backvec X_\tau = \tilde b_t(\backvec X_\tau)\dd\tau + g_t\dd \backvec W_\tau,
\qquad
\tilde b_t(x)= -\vec b_t(x)+g_t^2\nabla\log p_t(x),
\quad x\in\mathsf X_\cA.
\]
\noindent
\textbf{(Jump part: flux identity).}
The reverse jump kernel \(\backvec J_t\) is uniquely determined (for a.e. \(t\)) by the flux equation
\begin{equation}\label{eq:app-flux-restated}
p_t(\dd y)\backvec J_t(y,\dd x)=p_t(\dd x)\vec J_t(x,\dd y).
\end{equation}
In particular, since \(\vec J_t(a,\dd y)=\vec\lambda_t(a)R_t(\dd y\given a)\) for \(a\in\cA\) and \(\vec J_t(x,\cdot)\equiv 0\) for \(x\in\mathsf X_\cA\), we have for every anchor \(a\in\cA\) the measure identity on \(\mathsf X_\cA\):
\begin{equation}\label{eq:app-sticky-measure}
p_t(\dd y)\backvec J_t(y,\{a\}) = p_t(\{a\})\vec\lambda_t(a)R_t(\dd y\given a).
\end{equation}
If moreover \(p_t(\dd y)=p_t(y)\dd y\) on \(\mathsf X_\cA\) and \(R_t(\dd y\given a)=r_t(y\given a)\dd y\), then for \(y\in\mathsf X_\cA\),
\[
\backvec J_t(y,\{a\})=\frac{p_t(\{a\})}{p_t(y)}\vec\lambda_t(a)r_t(y\given a).
\]
Moreover \(\backvec J_t(a,\mathsf X_\cA)=0\), i.e. anchors are absorbing in reverse time.
\end{theorem*}

\begin{proof}
Fix \(t\in(0,T)\) for which the marginal \(p_t\) admits the decomposition
\[
p_t = p_t^{\mathrm{ac}} + p_t^{\mathrm{at}},
\qquad
p_t^{\mathrm{at}}=\sum_{a\in\cA} m_t(a)\delta_a,
\qquad
p_t^{\mathrm{ac}}(\dd x)=p_t(x)\dd x \ \text{on }\mathsf X_\cA,
\]
as in Assumption~\ref{ass:reg}(iii), and let \(\mathcal U = C_c^2(\R^d)\).

The proof extracts the backward generator \(\backvec L_t\) from a single integration-by-parts identity, taken from \citet{conforti2022timereversal}, by isolating the two pieces of structural information that identity carries. We first manipulate the jump contribution to recover the flux identity \eqref{eq:app-flux-restated}; specializing to the anchored forward kernel then yields the explicit form of \(\backvec J_t\) and the absorbing property at anchors. The diffusion contribution determines the drift \(\tilde b_t\) via distributional integration by parts; the boundary terms generated by the atomic part of \(p_t\) cancel against the backward sticky jumps, with the cancellation supplied by the flux identity itself.

\paragraph{The integration-by-parts identity.}
The forward generator on \(u\in\mathcal U\) is
\begin{equation}\label{eq:app-fwd-generator}
    \vec L_t u(x) = \mathds{1}_{\mathsf X_\cA}(x)\left[\vec b_t(x)\cdot\nabla u(x) + \tfrac12 g_t^2\Delta u(x)\right] + \int_{\R^d}\big(u(y)-u(x)\big)\vec J_t(x,\dd y),
\end{equation}
with \(\vec J_t(x,\cdot)\equiv 0\) on \(\mathsf X_\cA\) and \(\vec J_t(a,\dd y)=\vec\lambda_t(a)R_t(\dd y\given a)\) on \(\cA\); the corresponding carré du champ is
\[
\Gamma_t(u,v)(x)
= \mathds{1}_{\mathsf X_\cA}(x)g_t^2\nabla u(x)\cdot\nabla v(x)
+ \int_{\R^d}\big(u(y)-u(x)\big)\big(v(y)-v(x)\big)\vec J_t(x,\dd y).
\]
Under Assumption~\ref{ass:reg}, \(\mathcal U\) is a measure-determining algebra contained in \(\mathrm{dom}\vec L_t\), and \citet[Theorem~4.4]{conforti2022timereversal} ensures the existence of a backward generator \(\backvec L_t\) (defined \(p_t\)-a.e.) such that, for all \(u,v\in\mathcal U\) and a.e.\ \(t\),
\begin{equation}\label{eq:ibp-restated}
\int_{\R^d}\Big[\big(\vec L_t u + \backvec L_t u\big)(x)v(x) + \Gamma_t(u,v)(x)\Big]p_t(\dd x)=0.
\end{equation}
Equation~\eqref{eq:ibp-restated} is the only analytic input to the proof: the flux identity, the anchored jump formula, and the score-form drift will all be consequences of choosing test functions \(u,v\) appropriately. Conservation of the continuous-martingale quadratic variation under time reversal forces the reversed diffusion coefficient to coincide with \(g_t\), so \(\backvec L_t\) takes the form
\begin{equation}\label{eq:bwd-generator-ansatz}
\backvec L_t u(x)
= \mathds{1}_{\mathsf X_\cA}(x)\left[\tilde b_t(x)\cdot\nabla u(x) + \tfrac12 g_t^2\Delta u(x)\right]
+ \int_{\R^d}\big(u(y)-u(x)\big)\backvec J_t(x,\dd y),
\end{equation}
for an unknown drift \(\tilde b_t\) and an unknown backward jump kernel \(\backvec J_t\), both of which we now identify.

\paragraph{Step 1: flux identity.}\label{app:time-reversal-step1}
We first isolate the jump contribution to \eqref{eq:ibp-restated}. The diffusion parts of \(\vec L_t\), \(\backvec L_t\), and \(\Gamma_t\) satisfy a closed integration-by-parts identity on their own (the standard score-based time-reversal calculation in the absence of jumps; cf.\ \citet[Theorem~5.7]{conforti2022timereversal}) so subtracting it from \eqref{eq:ibp-restated} leaves the jump-only identity
\begin{equation}\label{eq:ibp-jump-only-restated}
\int_{\R^d}\Big[\big(\vec L_t^{\mathrm{jump}} u + \backvec L_t^{\mathrm{jump}} u\big)(x)v(x) + \Gamma_t^{\mathrm{jump}}(u,v)(x)\Big]p_t(\dd x)=0,
\end{equation}
where the superscript ``jump'' picks out the integral term of the corresponding operator. Expanding the three integrals,
\begin{align*}
    0 = & \iint v(x)\big(u(y)-u(x)\big)\vec J_t(x,\dd y)p_t(\dd x) + \iint v(x)\big(u(y)-u(x)\big)\backvec J_t(x,\dd y)p_t(\dd x) \\ 
    & {} + \iint \big(u(y)-u(x)\big)\big(v(y)-v(x)\big)\vec J_t(x,\dd y)p_t(\dd x).
\end{align*}
The algebraic identity
\[
    v(x)\big(u(y)-u(x)\big) + \big(u(y)-u(x)\big)\big(v(y)-v(x)\big) = \big(u(y)-u(x)\big)v(y)
\]
collapses the first and third terms of the right-hand side, leaving
\[
\iint v(x)\big(u(y)-u(x)\big)\backvec J_t(x,\dd y)p_t(\dd x)
= -\iint \big(u(y)-u(x)\big)v(y)\vec J_t(x,\dd y)p_t(\dd x).
\]
Swap the dummy variables \(x\leftrightarrow y\) on the right and use \(u(x)-u(y)=-\big(u(y)-u(x)\big)\):
\[
\iint v(x)\big(u(y)-u(x)\big)\backvec J_t(x,\dd y)p_t(\dd x)
= \iint v(x)\big(u(y)-u(x)\big)\vec J_t(y,\dd x)p_t(\dd y).
\]
The class of test functions \((x,y)\mapsto v(x)\big(u(y)-u(x)\big)\) for \(u,v\in\mathcal U\) is measure-determining off the diagonal, so the two measures on \(\R^d\times\R^d\) coincide:
\begin{equation}\label{eq:flux-derived}
p_t(\dd x)\backvec J_t(x,\dd y)=p_t(\dd y)\vec J_t(y,\dd x),
\end{equation}
which is exactly the flux identity \eqref{eq:app-flux-restated}.

\paragraph{Step 2: anchored jumps.}\label{app:time-reversal-step2}
We now specialize \eqref{eq:flux-derived} to the structure of the SJD forward kernel. By Assumption~\ref{ass:reg}(ii), \(\vec J_t(x,\cdot)\equiv 0\) for \(x\in\mathsf X_\cA\) and \(\vec J_t(a,\dd y)=\vec\lambda_t(a)R_t(\dd y\given a)\) for \(a\in\cA\). Fix \(a\in\cA\) and apply \eqref{eq:flux-derived} with \(\dd x = \{a\}\) and \(\dd y\subset\mathsf X_\cA\):
\begin{equation}\label{eq:flux-specialized}
p_t(\dd y)\backvec J_t(y,\{a\})
= p_t(\{a\})\vec J_t(a,\dd y)
= m_t(a)\vec\lambda_t(a)R_t(\dd y\given a),
\end{equation}
which is exactly \eqref{eq:app-sticky-measure}. If \(p_t\) admits density \(p_t(y)\) on \(\mathsf X_\cA\) and \(R_t(\dd y\given a)=r_t(y\given a)\dd y\), dividing both sides of \eqref{eq:flux-specialized} by \(p_t(y)\dd y\) gives the closed form
\[
\backvec J_t(y,\{a\}) = \frac{m_t(a)}{p_t(y)}\vec\lambda_t(a)r_t(y\given a),\qquad y\in\mathsf X_\cA,\ a\in\cA.
\]
Conversely, applying \eqref{eq:flux-derived} with \(\dd y=\{a\}\) and any \(\dd x\subset\mathsf X_\cA\),
\[
p_t(\{a\})\backvec J_t(a,\dd x) = p_t(\dd x)\vec J_t(x,\{a\}) = 0,
\]
since the forward kernel does not jump out of \(\mathsf X_\cA\). Hence \(\backvec J_t(a,\mathsf X_\cA)=0\): anchors are absorbing in reverse time.

The total backward sticky-jump rate at \(y\in\mathsf X_\cA\) is \(\backvec{\lambda_t^\star}(y) := \sum_{a\in\cA}\backvec J_t(y,\{a\})\). Summing \eqref{eq:flux-specialized} over \(a\) yields the identity
\begin{equation}\label{eq:flux-summed}
\sum_{a\in\cA} m_t(a)\vec\lambda_t(a)r_t(y\given a)
= p_t(y)\backvec{\lambda_t^\star}(y),\qquad y\in\mathsf X_\cA,
\end{equation}
which we will use in Step~3.

\paragraph{Step 3: diffusion drift.}\label{app:time-reversal-step3}
Take \(u,v\in\mathcal U\) with supports contained in \(\mathsf X_\cA\), so that \(u(a)=v(a)=0\) for every \(a\in\cA\). We catalogue the contributions of \eqref{eq:ibp-restated} term by term.

\emph{Forward generator.}
On \(\mathsf X_\cA\), \(\vec L_t u(x) = \vec b_t(x)\cdot\nabla u(x) + \tfrac12 g_t^2 \Delta u(x)\) since \(\vec J_t(x,\cdot)\equiv 0\). At \(x=a\in\cA\), the contribution \(\vec L_t u(a)v(a)\) vanishes because \(v(a)=0\).

\emph{Backward generator.}
On \(\mathsf X_\cA\), the ansatz \eqref{eq:bwd-generator-ansatz}, the result of Step~2, and \(u(a)=0\) give
\begin{align}\label{eq:bwd-generator}
\backvec L_t u(x)
&= \tilde b_t(x)\cdot\nabla u(x) + \tfrac12 g_t^2\Delta u(x)
+ \sum_{a\in\cA}\big(u(a)-u(x)\big)\backvec J_t(x,\{a\}) \nonumber \\
&= \tilde b_t\cdot\nabla u + \tfrac12 g_t^2\Delta u - u\backvec{\lambda_t^\star}.
\end{align}
The sticky jumps to anchors thus contribute the additional term \(-u(x)\backvec{\lambda_t^\star}(x)\) to \(\backvec L_t u\) on \(\mathsf X_\cA\), even though \(u\) is supported off the anchors. At \(x=a\), \(\backvec L_t u(a)v(a)=0\) since \(v(a)=0\).

\emph{Carré du champ.}
On \(\mathsf X_\cA\), \(\Gamma_t(u,v)(x) = g_t^2\nabla u\cdot\nabla v\) since \(\vec J_t(x,\cdot)\equiv 0\). At \(x=a\in\cA\), \(u(a)=v(a)=0\) leaves the nontrivial contribution
\[
\Gamma_t(u,v)(a) = \int_{\R^d} u(y)v(y)\vec\lambda_t(a)R_t(\dd y\given a).
\]

Combining the three contributions in \eqref{eq:ibp-restated},
\begin{multline}\label{eq:ibp-expanded-drift}
0 = \int_{\mathsf X_\cA}\Big[(\vec b_t+\tilde b_t)\cdot\nabla uv
+ g_t^2\Delta uv
+ g_t^2\nabla u\cdot\nabla v
- uv\backvec{\lambda_t^\star}\Big]p_t(x)\dd x \\
{}+ \sum_{a\in\cA} m_t(a)\vec\lambda_t(a)\int_{\mathsf X_\cA} u(y)v(y)r_t(y\given a)\dd y.
\end{multline}
The atomic carré-du-champ contribution and the backward sticky-jump contribution cancel exactly: by \eqref{eq:flux-summed},
\[
\sum_{a\in\cA} m_t(a)\vec\lambda_t(a)\int_{\mathsf X_\cA} uvr_t(y\given a)\dd y
= \int_{\mathsf X_\cA} u(y)v(y)\backvec{\lambda_t^\star}(y)p_t(y)\dd y,
\]
so the second line of \eqref{eq:ibp-expanded-drift} kills the \(-uv\backvec{\lambda_t^\star}\) term in the bulk integral. What remains is the standard diffusion-IbP identity
\begin{equation}\label{eq:diffusion-ibp}
\int_{\mathsf X_\cA}\Big[(\vec b_t+\tilde b_t)\cdot\nabla uv
+ g_t^2\Delta uv
+ g_t^2\nabla u\cdot\nabla v\Big]p_t(x)\dd x = 0.
\end{equation}
Integrating by parts in space (boundary terms vanish since \(u,v\) are compactly supported in \(\mathsf X_\cA\)),
\[
\int v\Delta up_t\dd x
= -\int \nabla u\cdot\nabla(v p_t)\dd x
= -\int \nabla u\cdot\big(p_t\nabla v + v\nabla p_t\big)\dd x,
\]
so the bracketed combination simplifies:
\[
\int g_t^2\big(v\Delta u + \nabla u\cdot\nabla v\big)p_t\dd x
= -g_t^2\int v\nabla u\cdot\nabla p_t\dd x.
\]
Therefore \eqref{eq:diffusion-ibp} reduces to
\[
0 = \int_{\mathsf X_\cA} v(x)A(x)\cdot\nabla u(x)\dd x,
\qquad
A(x):=p_t(x)\big(\vec b_t(x)+\tilde b_t(x)\big) - g_t^2\nabla p_t(x).
\]
A standard distributional argument forces \(A\equiv 0\) on \(\mathsf X_\cA\): for any \(\varphi\in C_c^\infty(\mathsf X_\cA)\) and any cutoff \(\chi\in C_c^\infty(\mathsf X_\cA)\) with \(\chi\equiv 1\) on \(\mathrm{supp}(\varphi)\), set \(u(x):=x_i\chi(x)\) and \(v(x):=\varphi(x)\), so that \(\nabla u=e_i\) on \(\mathrm{supp}(\varphi)\) and the identity above reads \(0 = \int_{\mathsf X_\cA}\varphi A_i\dd x\). As \(\varphi\) ranges over \(C_c^\infty(\mathsf X_\cA)\), \(A_i=0\) a.e.\ on \(\mathsf X_\cA\); ranging \(i\in\{1,\ldots,d\}\) yields \(A\equiv 0\) a.e. Dividing by \(p_t(x)>0\),
\begin{equation}\label{eq:reverse-drift-ibp}
\tilde b_t(x) = -\vec b_t(x) + g_t^2\nabla\log p_t(x),\qquad x\in\mathsf X_\cA,
\end{equation}
which is the claimed reverse drift, recovering the score-based time-reversal of \citet{song2021sde} on the off-anchor region.
\end{proof}

% \begin{corollary}[Probability-flow dynamics with sticky jumps]\label{cor:pfode}
% Let Assumption~\ref{ass:reg} hold and let $\backvec J_t$ be the reverse jump kernel of \Cref{thm:anchored-time-reversal}. Define the probability-flow process $(\widetilde X_\tau)_{\tau\in[0,T]}$, with $t = T-\tau$ and $\widetilde X_0 \sim p_T$, evolving on $\mathsf X_\cA$ by the ordinary differential equation
% \begin{equation}\label{eq:pf-ode}
% \dd \widetilde X_\tau = \Big({-\vec b_t(\widetilde X_\tau)} + \tfrac{1}{2}g_t^2\nabla\log p_t(\widetilde X_\tau)\Big)\dd\tau,
% \end{equation}
% and committing to anchors with the sticky-jump intensities of the exact reversal: from $y\in\mathsf X_\cA$, jumps into $a\in\cA$ occur at rate $\backvec J_t(y,\{a\})$ given by \eqref{eq:sticky-rate}, and anchors are absorbing. Then $\widetilde X_\tau \sim p_{T-\tau}$ for every $\tau\in[0,T]$; in particular, the probability-flow process and the exact reversal of \Cref{thm:anchored-time-reversal} share all time marginals.
% \end{corollary}

\begin{corollary}[Probability-flow dynamics with sticky jumps]\label{cor:pfode}
Let Assumption~\ref{ass:reg} hold and let $\backvec J_t$ be the reverse jump kernel of \Cref{thm:anchored-time-reversal}. Write
\[
\backvec b_t^{\mathrm{PF}}(x) := -\vec b_t(x) + \tfrac12 g_t^2\nabla\log p_t(x), \qquad x\in\mathsf X_\cA ,
\]
and assume in addition that $\backvec b_t^{\mathrm{PF}}$ is locally Lipschitz in $x$ on $\mathsf X_\cA$, that solutions of \eqref{eq:pf-ode} below exist on $[0,T]$ and remain in $\mathsf X_\cA$ for $p_T^{\mathrm{ac}}$-a.e.\ initial point, and that $\backvec{\lambda_t^\star}$ is locally bounded on $\mathsf X_\cA$. Define the probability-flow process $(\widetilde X_\tau)_{\tau\in[0,T]}$, with $t = T-\tau$ and $\widetilde X_0 \sim p_T$, evolving on $\mathsf X_\cA$ by the ordinary differential equation
\begin{equation}\label{eq:pf-ode}
\dd \widetilde X_\tau = \backvec b_t^{\mathrm{PF}}(\widetilde X_\tau)\,\dd\tau,
\end{equation}
and committing to anchors with the sticky-jump intensities of the exact reversal: from $y\in\mathsf X_\cA$, jumps into $a\in\cA$ occur at rate $\backvec J_t(y,\{a\})$ given by \eqref{eq:sticky-rate}, and anchors are absorbing. Then $\widetilde X_\tau \sim p_{T-\tau}$ for every $\tau\in[0,T]$; in particular, the probability-flow process and the exact reversal of \Cref{thm:anchored-time-reversal} share all time marginals.
\end{corollary}

\begin{proof}
Throughout, $t = T-\tau$. Set $\rho_\tau := p_{T-\tau}$, so that
\[
\rho_\tau = \rho_\tau^{\mathrm{ac}} + \sum_{a\in\cA} m_t(a)\,\delta_a ,
\]
and write $\rho_\tau(x) = p_t(x)$ for the density of $\rho_\tau^{\mathrm{ac}}$ on $\mathsf X_\cA$. Let $\backvec b_t := -\vec b_t + g_t^2\nabla\log p_t$ denote the reverse drift \eqref{eq:reverse-drift-ibp}. Spatial derivatives of $\rho_\tau$ are understood in the sense of distributions on $\mathsf X_\cA$.

\paragraph{Step 1: Fokker--Planck equation of the reversal on $\mathsf X_\cA$.}
Let $u\in C_c^2(\mathsf X_\cA)$. Since $u$ vanishes on $\cA$, \eqref{eq:bwd-generator} gives, for $x\in\mathsf X_\cA$,
\[
\backvec L_t u(x)
= \backvec b_t(x)\cdot\nabla u(x) + \tfrac12 g_t^2\Delta u(x) + \sum_{a\in\cA}\big(u(a)-u(x)\big)\backvec J_t(x,\{a\})
= \backvec b_t\cdot\nabla u + \tfrac12 g_t^2\Delta u - \backvec{\lambda_t^\star}\,u ,
\]
while $\backvec L_t u(a) = 0$ for $a\in\cA$, since the diffusion part vanishes off $\mathsf X_\cA$ and $\backvec J_t(a,\cdot)\equiv 0$ (anchors are absorbing). By \Cref{thm:anchored-time-reversal}, $\backvec X$ is Markov with generator $\backvec L_t$ and $\backvec X_\tau\sim\rho_\tau$, so
\[
\frac{\dd}{\dd\tau}\,\mathbb E\big[u(\backvec X_\tau)\big] = \mathbb E\big[(\backvec L_t u)(\backvec X_\tau)\big].
\]
Because $u(a) = 0$ and $\backvec L_t u(a) = 0$, the atoms of $\rho_\tau$ contribute to neither side, and
\begin{equation}\label{eq:pf-weak}
\frac{\dd}{\dd\tau}\int_{\mathsf X_\cA} u\,\rho_\tau\,\dd x = \int_{\mathsf X_\cA} (\backvec L_t u)\,\rho_\tau\,\dd x .
\end{equation}
Integrating by parts in the drift term and applying Green's second identity to the Laplacian term (there are no boundary terms, since $\operatorname{supp}u$ is compact in $\mathsf X_\cA$),
\[
\int_{\mathsf X_\cA} (\backvec L_t u)\,\rho_\tau\,\dd x = \int_{\mathsf X_\cA} u\,\backvec L_t^\star\rho_\tau\,\dd x,
\qquad
\backvec L_t^\star v := -\nabla\cdot\big(\backvec b_t\, v\big) + \tfrac12 g_t^2\Delta v - \backvec{\lambda_t^\star}\,v ,
\]
where $\backvec L_t^\star$ is the $L^2(\dd x)$ adjoint of $\backvec L_t$ on $C_c^2(\mathsf X_\cA)$. As \eqref{eq:pf-weak} holds for every $u\in C_c^2(\mathsf X_\cA)$,
\begin{equation}\label{eq:pf-rfp}
\partial_\tau\rho_\tau = \backvec L_t^\star\rho_\tau
= -\nabla\cdot\big(\backvec b_t\,\rho_\tau\big) + \tfrac12 g_t^2\Delta\rho_\tau - \backvec{\lambda_t^\star}\,\rho_\tau
\quad\text{on } \mathsf X_\cA .
\end{equation}

\paragraph{Step 2: Continuity-equation form.}
Since $\rho_\tau\nabla\log p_t = \nabla\rho_\tau$, we have $\nabla\cdot\big(g_t^2\rho_\tau\nabla\log p_t\big) = g_t^2\Delta\rho_\tau$. Substituting $\backvec b_t = \backvec b_t^{\mathrm{PF}} + \tfrac12 g_t^2\nabla\log p_t$ into \eqref{eq:pf-rfp} therefore gives
\begin{equation}\label{eq:pf-ac}
\partial_\tau\rho_\tau = -\nabla\cdot\big(\backvec b_t^{\mathrm{PF}}\rho_\tau\big) - \backvec{\lambda_t^\star}\,\rho_\tau
\quad\text{on } \mathsf X_\cA .
\end{equation}

\paragraph{Step 3: Anchor masses.}
Since the forward process does not re-attach to anchors, $m_t(a) = p_0(a)S_a(t)$, so $\partial_t m_t(a) = -\vec\lambda_t(a)\,m_t(a)$. With $\rho_\tau(\{a\}) = m_{T-\tau}(a)$,
\begin{equation}\label{eq:pf-at}
\partial_\tau\rho_\tau(\{a\}) = \vec\lambda_t(a)\,m_t(a) = \int_{\mathsf X_\cA}\backvec J_t(y,\{a\})\,\rho_\tau(y)\,\dd y ,
\end{equation}
where the last equality uses \eqref{eq:sticky-rate} with $\rho_\tau(y) = p_t(y)$, together with $\int_{\mathsf X_\cA} r_t(y\mid a)\,\dd y = 1$.

\paragraph{Step 4: Law of the probability-flow process.}
Let $q_\tau$ denote the law of $\widetilde X_\tau$, with absolutely continuous part $q_\tau^{\mathrm{ac}}$ on $\mathsf X_\cA$ and atoms $q_\tau(\{a\})$. Between jumps, $\widetilde X$ follows \eqref{eq:pf-ode}; from $y\in\mathsf X_\cA$ it jumps to $a$ at rate $\backvec J_t(y,\{a\})$, hence leaves $\mathsf X_\cA$ at total rate $\backvec{\lambda_t^\star}(y)$; and anchors are absorbing. Its generator is therefore
\[
L_t^{\mathrm{PF}} u(x) = \mathbf 1_{\mathsf X_\cA}(x)\Big[\backvec b_t^{\mathrm{PF}}(x)\cdot\nabla u(x) + \sum_{a\in\cA}\big(u(a)-u(x)\big)\backvec J_t(x,\{a\})\Big].
\]
Repeating Step~1 with $L_t^{\mathrm{PF}}$ in place of $\backvec L_t$ (there is no Laplacian term, as \eqref{eq:pf-ode} carries no noise) shows that $q^{\mathrm{ac}}$ satisfies \eqref{eq:pf-ac}; taking $u = \mathbf 1_{\{a\}}$, for which $L_t^{\mathrm{PF}}u = \backvec J_t(\cdot,\{a\})$ on $\mathsf X_\cA$ and $L_t^{\mathrm{PF}}u(a) = 0$, shows that $q$ satisfies \eqref{eq:pf-at}. Moreover $q_0 = p_T = \rho_0$.

\paragraph{Step 5: Uniqueness.}
The coefficients $\backvec b_t^{\mathrm{PF}}$ and $\backvec{\lambda_t^\star}$ of \eqref{eq:pf-ac} are fixed functions determined by $p_t$, so \eqref{eq:pf-ac} is a linear transport equation with a killing term. Under the stated regularity it admits a unique nonnegative solution for a given initial datum, namely the one given by the method of characteristics along the flow of \eqref{eq:pf-ode}. Since $\rho^{\mathrm{ac}}$ and $q^{\mathrm{ac}}$ both solve \eqref{eq:pf-ac} with initial datum $p_T^{\mathrm{ac}}$, we get $q_\tau^{\mathrm{ac}} = \rho_\tau^{\mathrm{ac}}$ for all $\tau\in[0,T]$. The right-hand sides of \eqref{eq:pf-at} then coincide for $q$ and $\rho$, and since $q_0(\{a\}) = \rho_0(\{a\}) = m_T(a)$, integrating in $\tau$ gives $q_\tau(\{a\}) = \rho_\tau(\{a\})$ for every $a\in\cA$. Hence $q_\tau = \rho_\tau = p_{T-\tau}$ for all $\tau\in[0,T]$.

Equivalently, in forward time \eqref{eq:pf-ac} reads
\[
\partial_t p_t = -\nabla\cdot\Big[\Big(\vec b_t - \tfrac12 g_t^2\nabla\log p_t\Big)p_t\Big] + \backvec{\lambda_t^\star}\,p_t
\quad\text{on } \mathsf X_\cA .
\]
The term $+\backvec{\lambda_t^\star}p_t$ is the source created by unsticking: mass leaves the anchors and enters $\mathsf X_\cA$. In reverse time it becomes the sink $-\backvec{\lambda_t^\star}\rho_\tau$: mass leaves $\mathsf X_\cA$ through sticky jumps, and \eqref{eq:pf-at} deposits it back on the anchors.
\end{proof}

\begin{remark}[Guidance]\label{rem:guidance} % FLAG: label
Because every reverse quantity is an analytic functional of the single classifier $P_\theta$,
conditional guidance requires no additional machinery. For a conditioning signal $c$ and a
guidance weight $w$, set
\[
  P^{w}_\theta(\cdot\mid y,t) \propto
  P_\theta(\cdot\mid y,t,c)^{w}P_\theta(\cdot\mid y,t,\varnothing)^{1-w},
\]
a mixture of logits, and substitute $P^{w}_\theta$ for $P_\theta$ in the induced
score (\eqref{eq:app-classifier-score} and \eqref{eq:blended-score}) and in the plug-ins~\eqref{eq:plugin-estimators}. One guided posterior then steers the drift, the commit rate, and the destination jointly. \citet{dieleman2022cdcd} demonstrate the continuous half of this construction. Our experiments are unconditional, so we state the mechanism without an empirical claim.
\end{remark}

\subsection{\texorpdfstring{Proof of \Cref{thm:dhm}}{Proof of Theorem \ref{thm:dhm}}}\label{app:dhm}

We restate \Cref{thm:dhm} with the relevant SJD-specific objects in scope, then give the proof.

\begin{theorem*}[denoising hazard matching, restated]
Let $\vec X_0\sim p_0$ on $\cA$ and, conditional on $\vec X_0$, let $\vec X_t$ be the SJD state at time $t$ restricted to the continuous region $\mathsf X_\cA$, with conditional density $p_t^{\mathrm{ac}}(\cdot\given \vec X_0)$ given by the unstick-time mixture of \Cref{rem:mixture}. For $y\in\mathsf X_\cA$ and $x_0\in\cA$, define the DHM target
\begin{equation}\label{eq:app-dhm-target}
\widehat\lambda_t(y, x_0) := \frac{\vec\lambda_t(x_0)S_{x_0}(t)r_{x_0}(y, t)}{p_t^{\mathrm{ac}}(y\given x_0)},
\end{equation}
with $S_{x_0}(t)=\exp\big(-\int_0^t \vec\lambda_s(x_0)\dd s\big)$. Then for almost every $y\in\mathsf X_\cA$,
\begin{equation}\label{eq:app-dhm-identity}
\E\big[\widehat\lambda_t(\vec X_t, \vec X_0)\bigm| \vec X_t = y\big] = \backvec{\lambda_t^\star}(y),
\end{equation}
where $\backvec{\lambda_t^\star}(y)$ is the total reverse jump rate of \eqref{eq:true-total-intensity}.
\end{theorem*}

\begin{proof}
The proof is a direct application of Bayes' rule for the SJD posterior. The crucial point is that the $p_t^{\mathrm{ac}}(y\given a)$ in the denominator of \eqref{eq:app-dhm-target} is the actual conditional density of the SJD continuous part and not an auxiliary corruption kernel. The same object therefore reappears in the denoising posterior \eqref{eq:denoising-posterior}, and the two factors cancel.

Because the SJD forward dynamics never re-attach to anchors, the only way for a particle to be at anchor $a$ at time $t$ is to have started at $a$ and never unstuck. Hence
\begin{equation}\label{eq:app-mass-split}
m_t(a) = p_0(a)S_a(t),\qquad a\in\cA.
\end{equation}

Let $y\in\mathsf X_\cA$. Since $\vec X_0$ takes values in the finite set $\cA$, the joint density of $(\vec X_0,\vec X_t)$ on $\cA\times\mathsf X_\cA$ is $p_0(a)p_t^{\mathrm{ac}}(y\given a)$, and the marginal continuous density of $\vec X_t$ is
\[
p_t(y) = \sum_{a\in\cA} p_0(a)p_t^{\mathrm{ac}}(y\given a),\qquad y\in\mathsf X_\cA.
\]
Bayes' rule gives the denoising posterior \eqref{eq:denoising-posterior},
\begin{equation}\label{eq:app-bayes-posterior}
p_t(a\given y) = \Pr\big(\vec X_0 = a\bigm| \vec X_t = y\big) = \frac{p_0(a)p_t^{\mathrm{ac}}(y\given a)}{p_t(y)}.
\end{equation}

Because $\vec X_0$ is supported on $\cA$, the conditional expectation in \eqref{eq:app-dhm-identity} is a finite sum:
\begin{equation}\label{eq:app-dhm-discrete}
\E\big[\widehat\lambda_t(\vec X_t, \vec X_0)\bigm| \vec X_t = y\big] = \sum_{a\in\cA} \widehat\lambda_t(y, a)p_t(a\given y).
\end{equation}

Substituting \eqref{eq:app-dhm-target} and \eqref{eq:app-bayes-posterior} into \eqref{eq:app-dhm-discrete},
\[
\sum_{a\in\cA}\widehat\lambda_t(y, a)p_t(a\given y)
=
\sum_{a\in\cA}\frac{\vec\lambda_t(a)S_a(t)r_a(y, t)}{p_t^{\mathrm{ac}}(y\given a)}\cdot\frac{p_0(a)p_t^{\mathrm{ac}}(y\given a)}{p_t(y)}.
\]
The factor $p_t^{\mathrm{ac}}(y\given a)$ cancels exactly:
\[
=\frac{1}{p_t(y)}\sum_{a\in\cA} p_0(a)S_a(t)\vec\lambda_t(a)r_a(y, t).
\]
Applying the mass identity \eqref{eq:app-mass-split},
\[
=\frac{1}{p_t(y)}\sum_{a\in\cA} m_t(a)\vec\lambda_t(a)r_a(y, t)
=\backvec{\lambda_t^\star}(y),
\]
where the last equality is the definition of the total reverse jump rate \eqref{eq:true-total-intensity}.
This proves \eqref{eq:app-dhm-identity}.
\end{proof}

\subsection{Rao--Blackwellized plug-in}\label{app:rb-plugin}

\Cref{thm:dhm} expresses the total reverse hazard $\backvec{\lambda_t^\star}(y)$ as a conditional expectation under the SJD forward pairing. We now record three properties that justify the plug-in estimators \eqref{eq:plugin-estimators}: a per-anchor flux identity, consistency of the plug-in when the classifier matches the true posterior, and the Rao--Blackwell variance reduction.

\paragraph{Setup.}
Throughout this subsection we fix $t\in(0,T]$ and $y\in\mathsf X_\cA$, and write $\widehat\lambda_t(y,a)$ for the DHM target \eqref{eq:app-dhm-target} and $p_t(a\given y)$ for the SJD denoising posterior \eqref{eq:app-bayes-posterior}. We assume the support condition $r_a(y,t)>0\Rightarrow p_t^{\mathrm{ac}}(y\given a)>0$, which is automatic whenever the convolution semigroup $r_\tau * K_{\tau\to t}$ has full support on $\mathsf X_\cA$ (in particular for the VP-matched family of \Cref{ex:vp-matched}). Given any probability vector $P_\theta(\cdot\given y, t)$ on $\cA$, recall the plug-in estimators
\begin{equation}\label{eq:app-plugin-restated}
    \Lambda_{\mathrm{plug},\theta}(a\given y, t) := \widehat\lambda_t(y, a)P_\theta(a\given y, t),\qquad \lambda_{\mathrm{plug},\theta}(y, t) := \sum_{a\in\cA}\Lambda_{\mathrm{plug},\theta}(a\given y, t).
\end{equation}

\paragraph{Per-anchor flux identity.}
The following identity 
%behind \Cref{rem:reweight-irreducible}
is the per-anchor analogue of \Cref{thm:dhm}: the true reverse \emph{per-anchor} intensity factors exactly as the DHM target times the denoising posterior.

\begin{lemma}[Per-anchor identity]\label{lem:per-anchor}
For every $a\in\cA$ and almost every $y\in\mathsf X_\cA$,
\begin{equation}\label{eq:app-per-anchor-identity}
    \Lambda_t^\star(a\given y) = \widehat\lambda_t(y, a)p_t(a\given y).
\end{equation}
\end{lemma}

\begin{proof}
Substitute \eqref{eq:app-dhm-target} for $\widehat\lambda_t(y, a)$ and \eqref{eq:app-bayes-posterior} for $p_t(a\given y)$:
\[
\widehat\lambda_t(y, a)p_t(a\given y)
=
\frac{\vec\lambda_t(a)S_a(t)r_a(y, t)}{p_t^{\mathrm{ac}}(y\given a)}\cdot\frac{p_0(a)p_t^{\mathrm{ac}}(y\given a)}{p_t(y)}
=
\frac{p_0(a)S_a(t)\vec\lambda_t(a)r_a(y, t)}{p_t(y)}.
\]
Using $m_t(a)=p_0(a)S_a(t)$ from \eqref{eq:app-mass-split}, the right-hand side is exactly $\Lambda_t^\star(a\given y)$ as defined in \eqref{eq:true-per-anchor-intensity}.
\end{proof}

\Cref{lem:per-anchor} is the structural fact that makes a single classifier sufficient: the DHM target $\widehat\lambda_t(y, a)$ is precisely the Radon--Nikodym factor that converts a probability mass $p_t(a\given y)$ (a posterior over data labels) into a per-anchor jump rate $\Lambda_t^\star(a\given y)$ (jumps per unit time into anchor $a$).

\paragraph{Consistency of the plug-in.}
\Cref{lem:per-anchor} immediately yields consistency: substituting the true posterior into the plug-in recovers the true reverse intensities.

\begin{proposition}[Consistency]\label{prop:plugin-consistent}
If $P_\theta(a\given y, t) = p_t(a\given y)$ for every $a\in\cA$ and almost every $y\in\mathsf X_\cA$, then
\begin{equation}\label{eq:app-consistency}
\Lambda_{\mathrm{plug},\theta}(a\given y, t) = \Lambda_t^\star(a\given y),
\qquad
\lambda_{\mathrm{plug},\theta}(y, t) = \backvec{\lambda_t^\star}(y),
\qquad
\pi_{\mathrm{plug},\theta}(a\given y, t) = \pi_t^\star(a\given y).
\end{equation}
\end{proposition}

\begin{proof}
The per-anchor equality is \Cref{lem:per-anchor}. The total-rate equality follows by summing over $a$, and the allocation equality by normalizing.
\end{proof}

\Cref{prop:plugin-consistent} pins down what the classifier needs to learn: the cross-entropy loss \eqref{eq:classifier-loss} is a strictly proper scoring rule whose Bayes optimum is exactly the posterior $p_t(\cdot\given y)$, and the plug-in inherits all three identities of \eqref{eq:app-consistency} as soon as $P_\theta$ matches that posterior.

\paragraph{Rao--Blackwellization.}
The DHM identity \eqref{eq:app-dhm-identity} also yields an unbiased Monte Carlo estimator of $\backvec{\lambda_t^\star}(y)$ from a single $(\vec X_0, \vec X_t)$ pair: the random variable
\begin{equation}\label{eq:app-Z}
    Z := \widehat\lambda_t(\vec X_t, \vec X_0).
\end{equation}
By \Cref{thm:dhm}, $\E[Z\given \vec X_t = y] = \backvec{\lambda_t^\star}(y)$. The plug-in \eqref{eq:app-plugin-restated} replaces $Z$ by its conditional expectation given $(\vec X_t, t)$, with the unknown posterior approximated by $P_\theta$. The next proposition makes this Rao--Blackwell relationship precise.

\begin{proposition}[Rao--Blackwellization]\label{prop:rb}
Let $Z = \widehat\lambda_t(\vec X_t, \vec X_0)$ as in \eqref{eq:app-Z}, with $(\vec X_0, \vec X_t)$ paired by the SJD forward law. Assume $\E[Z^2\given t] < \infty$. Then:
\begin{enumerate}
    \item \emph{(Conditional mean.)} $\E[Z\given \vec X_t, t] = \backvec{\lambda_t^\star}(\vec X_t)$ almost surely.
    \item \emph{(Variance decomposition.)} The law of total variance gives
    \begin{equation}\label{eq:app-rb-decomp}
        \Var(Z\given t) = \Var\big(\backvec{\lambda_t^\star}(\vec X_t)\bigm| t\big) + \E\big[\Var(Z\given \vec X_t, t)\bigm| t\big].
    \end{equation}
    \item \emph{(Variance reduction.)} Consequently,
    \begin{equation}\label{eq:app-rb-inequality}
        \Var\big(\backvec{\lambda_t^\star}(\vec X_t)\bigm| t\big) \le \Var(Z\given t),
    \end{equation}
    with equality if and only if $Z$ is almost surely a function of $(\vec X_t, t)$, i.e.\ $\Var(Z\given \vec X_t, t)=0$ a.s.
\end{enumerate}
\end{proposition}

\begin{proof}
For (i), apply the discrete expansion \eqref{eq:app-dhm-discrete} with $y = \vec X_t$:
\[
\E[Z\given \vec X_t, t] = \sum_{a\in\cA}\widehat\lambda_t(\vec X_t, a)p_t(a\given \vec X_t) = \backvec{\lambda_t^\star}(\vec X_t),
\]
where the second equality is \Cref{thm:dhm}. Statement (ii) is the law of total variance, conditional on $t$, with conditioning $\sigma$-field $\sigma(\vec X_t, t)$ and the substitution $\E[Z\given \vec X_t, t] = \backvec{\lambda_t^\star}(\vec X_t)$ from (i). Statement (iii) follows from \eqref{eq:app-rb-decomp} since the second term is nonnegative, with equality iff that term vanishes.
\end{proof}

\paragraph{Interpretation.}
The single-sample target $Z$ in \eqref{eq:app-Z} is unbiased for $\backvec{\lambda_t^\star}(\vec X_t)$ but carries the within-cell noise $\Var(Z\given \vec X_t, t)$: at a fixed $\vec X_t = y$, different latent draws of $\vec X_0\sim p_t(\cdot\given y)$ produce different values $\widehat\lambda_t(y, \vec X_0)$. The Rao--Blackwellized estimator $\sum_a \widehat\lambda_t(y, a)p_t(a\given y)$ averages this noise out exactly, replacing $Z$ by its conditional mean and saving the second term in \eqref{eq:app-rb-decomp}. The plug-in \eqref{eq:app-plugin-restated} performs this Rao--Blackwellization with the true posterior replaced by the trained classifier $P_\theta$:
\begin{equation}\label{eq:app-rb-plugin-formula}
    \lambda_{\mathrm{plug},\theta}(y, t) = \underbrace{\sum_{a\in\cA}\widehat\lambda_t(y, a)p_t(a\given y)}_{=\backvec{\lambda_t^\star}(y)\text{ by \Cref{thm:dhm}}} + \underbrace{\sum_{a\in\cA}\widehat\lambda_t(y, a)\big(P_\theta(a\given y, t) - p_t(a\given y)\big)}_{\text{classifier error}}.
\end{equation}
The first term is exact by Rao--Blackwell; the residual is controlled by the gap between $P_\theta$ and the Bayes-optimal posterior, which the cross-entropy loss \eqref{eq:classifier-loss} drives to zero in the well-specified, infinite-data limit. In particular, by Cauchy--Schwarz,
\[
    \big|\lambda_{\mathrm{plug},\theta}(y, t) - \backvec{\lambda_t^\star}(y)\big| \le \Big(\sum_{a\in\cA}\widehat\lambda_t(y, a)^2\Big)^{1/2}\Big(\sum_{a\in\cA}\big(P_\theta(a\given y, t) - p_t(a\given y)\big)^2\Big)^{1/2},
\]
so any consistent classifier yields a consistent plug-in hazard, and the conversion from posterior error to hazard error is mediated by the analytic factor $\widehat\lambda_t(y, \cdot)$.

Two practical consequences follow. First, training reduces to a single cross-entropy objective on a discrete classifier, with no auxiliary regression head for the (unbounded, heavy-tailed) hazard target $Z$. Second, all SJD-specific structure---the forward hazard, the unsticking density, and the unstick-time mixture $p_t^{\mathrm{ac}}(y\given a)$---enters only through the analytic reweighting $\widehat\lambda_t(y, a)$, which is computed once per reverse step using the closed-form quadrature of \Cref{app:plugin}. The classifier never sees these quantities directly; it learns only the discrete posterior $p_t(\cdot\given y)$, and the analytic reweighting converts that into the rates and intensities required by the reverse SDE.

\subsection{\texorpdfstring{Proof of \Cref{prop:plain-hybrid}}{Proof of the plain-hybrid proposition}}\label{app:plain-hybrid}
\textbf{\begin{proposition*}[The plain hybrid coincides with CADD, restated]
Consider the VP-matched SJD of \Cref{ex:vp-matched} at $(W,\eta)=(I,1)$ with anchor-agnostic hazard $\vec\lambda(t)$ and survival $S(t)$, and CADD with masking survival $\alpha_t$, continuous kernel $\cN(\sqrt{\bar\gamma_t}E(x_0),(1-\bar\gamma_t)I_d)$ at masked positions, and a common embedding map $E$. Identify $S(t)\leftrightarrow\alpha_t$ and $(\alpha(t),\sigma^2(t))\leftrightarrow(\sqrt{\bar\gamma_t},1-\bar\gamma_t)$. Then:
\begin{enumerate}[label=(\roman*),leftmargin=*,nosep]
    \item the time-$t$ laws of the corrupted state given the data coincide for every $t$, and the SJD cross-entropy objective of \Cref{alg:sjd-training} equals CADD's cross-entropy objective;
    \item the reverse process of \Cref{thm:anchored-time-reversal} reduces to a commit rate $\vec\lambda(t)S(t)/(1-S(t))$ shared across anchors, an allocation equal to the classifier posterior, and a reverse SDE whose score is the residual to the posterior-mean embedding;
    \item CADD's sampler is a first-order discretization of (ii), so the two generative laws coincide in the small-step limit.
\end{enumerate}
\end{proposition*}}

\paragraph{(i) Forward laws and objectives.}
Fix a position and condition on $\vec X_0$. Under $S(t)\leftrightarrow\alpha_t$ the stuck probability equals the survival of CADD's absorbing chain. Conditionally on having unstuck, \eqref{eq:eta1-collapse} gives the off-anchor location the law $\cN(\alpha(t)a,\sigma^2(t)I_d)$, which is CADD's masked-position latent $\cN(\sqrt{\bar\gamma_t}E(a),(1-\bar\gamma_t)I_d)$ under $(\alpha(t),\sigma^2(t))\leftrightarrow(\sqrt{\bar\gamma_t},1-\bar\gamma_t)$. Positions are conditionally independent given $\vec X_0$ in both constructions, so the time-$t$ joint laws agree. The state map is a measurable bijection (a stuck coordinate sits exactly on its anchor and thus carries the token identity, as does an unmasked token; an unstuck coordinate carries the latent), so the Bayes posteriors over the clean sequence agree, and both objectives draw $t$ uniformly and apply unweighted cross-entropy at unstuck (masked) positions: \Cref{alg:sjd-training} coincides with CADD's loss.

\paragraph{(ii) The reverse process at $(I,1)$.}
At $\eta=1$, $v_t(\tau)=\sigma^2(t)$ by \eqref{eq:vp-matched-convolution}, so the spatial factor of the DHM target equals $r_t/p_t^{\mathrm{ac}}=1/(1-S(t))$ by \eqref{eq:eta1-collapse}, giving $\widehat\lambda_t(y,a)=\vec\lambda(t)S(t)/(1-S(t))$ for every $a$ and $y$. The plug-in total rate \eqref{eq:plugin-estimators} is therefore anchor-independent and the allocation \eqref{eq:plugin-allocation} reduces to $P_\theta(\cdot\given y,t)$. Since $v_t(\tau)$ no longer depends on $\tau$, the time-posterior weight in \eqref{eq:app-classifier-score} integrates to one and the classifier-induced score reduces to
\[
s_\theta(y,t) = -\frac{y-\alpha(t)m_\theta(y,t)}{\sigma^2(t)}, \qquad m_\theta(y,t):=\sum_{a\in\cA}P_\theta(a\given y,t)a,
\]
the residual to the posterior-mean anchor, exact at the cross-entropy optimum.

\paragraph{(iii) CADD's sampler as a discretization.}
Over a reverse step from $t$ to $t-\Delta$, the exponential clock of \Cref{alg:sjd-sampling} commits with probability $1-\exp(-\widehat\lambda_t\Delta)=\vec\lambda(t)S(t)\Delta/(1-S(t))+O(\Delta^2)$; since $\dot S=-\vec\lambda S$, this matches CADD's unmasking probability $(\alpha_{t-\Delta}-\alpha_t)/(1-\alpha_t)$ to first order under $S\leftrightarrow\alpha$. The committed token is drawn from the classifier posterior in both. While uncommitted, CADD's update $z_{t-\Delta}\sim\cN(\tilde\mu_t(\hat z_{0,\theta},z_t),\tilde\beta_t I_d)$ with $\hat z_{0,\theta}=\sum_a P_\theta(a\given\cdot)\,E(a)=m_\theta$ is the standard DDPM ancestral step, which agrees to first order in $\Delta$ with the Euler--Maruyama step of \Cref{alg:sjd-sampling} on the reverse SDE \eqref{eq:backwardSDE} with the score above \citep{ho2020ddpm,song2021sde}. Upon commitment both processes place the coordinate at the committed anchor and freeze it. The two samplers are therefore first-order discretizations of the same jump-diffusion, and their laws coincide as $\Delta\to 0$. CADD's alternative latent estimators, such as resampling $z_t$ from the forward kernel centered at $\hat z_{0,\theta}$, are further approximations of the same reverse law, trading fidelity for mode coverage as \citet{zheng2025cadd} discuss. \qed

\subsection{Learning the forward hazard, and why we fix it}\label{app:e2e-hazard}
\Cref{sec:conn-schedules} identifies the per-anchor forward hazard as a further design axis. Here we derive a variational objective and analyze learning the hazard jointly with the denoising classifier. The matched MD4 comparison is reported in \Cref{app:md4-likelihood}, and learned and fixed-hazard experiments are reported in \Cref{app:text8-hazards}.

\paragraph{Objective.} Parametrize $\vec\lambda_t(a)=\beta(t)w(a)$ with $w:\cA\to\R_{>0}$ learned jointly with $\theta$ (the anchor-agnostic baseline is $w\equiv1$). The plain cross-entropy \eqref{eq:classifier-loss} is not a bound on $\log p(\vec X_0)$, so its $w$-gradient does not target likelihood; we replace it with the negative ELBO. At $\eta=1$ the unstick-time mixture collapses \eqref{eq:eta1-collapse}, making the DHM target \eqref{eq:dhm-target} the closed-form scalar $\widehat\lambda_t(a)=\beta(t)w(a)S_a(t)/(1-S_a(t))$ and the conditional score target $w$-free.

\begin{proposition}[SJD negative ELBO at $\eta=1$]\label{prop:elbo-eta1}
Let $S_a(t)=e^{-w(a)B(t)}$ with $B(t)=\int_0^t\beta(s)\dd s$, $m_t(a)=p_0(a)S_a(t)$, and assume $S_a(T)=0$ for every $a$ together with the terminal conditions $\alpha(T)=0$, $\sigma(T)=1$. Then the negative ELBO decomposes as $\mathcal{L}(\theta,w)=\mathcal{L}_{\mathrm{CE}}+\mathcal{L}_{\mathrm{RB}}+\mathcal{L}_{\mathrm{score}}$, where
\begin{align}
\mathcal{L}_{\mathrm{CE}} &= \int_0^T\sum_{a\in\cA}\big(-\dot m_t(a)\big)\E_{Y\sim q_t(\cdot\given a)}\big[-\log P_\theta(a\given Y,t)\big]\dd t,\label{eq:lce}\\
\mathcal{L}_{\mathrm{RB}} &= \int_0^T\E_{Y\sim p_t^{\mathrm{ac}}}\Big[\textstyle\sum_{a\in\cA}\widehat\lambda_t(a)\big(P_\theta(a\given Y,t)-p_t(a\given Y)\big)\Big]\dd t,\label{eq:lrb}\\
\mathcal{L}_{\mathrm{score}} &= \int_0^T\sum_{a\in\cA}p_0(a)\big(1-S_a(t)\big)\E_{Y\sim q_t(\cdot\given a)}\Big[\tfrac{g_t^2}{2}\big\|s_\theta(Y,t)+\tfrac{Y-\alpha(t)a}{\sigma^2(t)}\big\|^2\Big]\dd t,\label{eq:lscore}
\end{align}
$-\dot m_t(a)=p_0(a)\beta(t)w(a)S_a(t)$ is the forward rate of mass leaving anchor $a$, and $\E_{Y\sim p_t^{\mathrm{ac}}}$ denotes integration against $p_t^{\mathrm{ac}}$ on $\mathsf X_\cA$.
\end{proposition}

\begin{proof}
Condition on $\vec X_0=a$. At $\eta=1$ the conditional law at time $t$ is the atom $S_a(t)\delta_a$ plus the absolutely continuous part $(1-S_a(t))q_t(\cdot\given a)$ \eqref{eq:eta1-collapse}. The conditional process is itself an SJD, so the flux identity of \Cref{thm:anchored-time-reversal} applies to it: the conditional reverse intensity into $a$ is the state-free scalar $S_a(t)\vec\lambda_t(a)q_t(y\given a)\big/\big[(1-S_a(t))q_t(y\given a)\big]=\widehat\lambda_t(a)$ (the DHM target \eqref{eq:dhm-target} is precisely the $\vec X_0$-conditional reverse intensity), and the conditional reverse drift uses the conditional score $-(y-\alpha(t)a)/\sigma^2(t)$. The model reverses with score $s_\theta$ and plug-in rates $\Lambda_\theta(b\given y,t)=\widehat\lambda_t(b)P_\theta(b\given y,t)$ \eqref{eq:plugin-estimators}; anchors are absorbing on both sides, so committed segments contribute nothing. The pathwise KL between the two reverse laws is the diffusion quadratic term plus the jump term in Bregman form, following \citet[Prop.~2]{campbell2023transdimensional},
$\sum_b\Lambda_\theta(b)-\widehat\lambda_t(a)\log\Lambda_\theta(a)+\widehat\lambda_t(a)\log\widehat\lambda_t(a)-\widehat\lambda_t(a)$,
integrated over the off-anchor event, plus a terminal KL that vanishes under the stated assumptions. Substituting $\Lambda_\theta$, the $\log\widehat\lambda_t(a)$ terms cancel, leaving $\sum_b\widehat\lambda_t(b)P_\theta(b\given Y,t)-\widehat\lambda_t(a)\log P_\theta(a\given Y,t)-\widehat\lambda_t(a)$. Average over $a\sim p_0$ and the off-anchor event (weight $1-S_a(t)$, $Y\sim q_t(\cdot\given a)$). The middle term integrates to $\mathcal{L}_{\mathrm{CE}}$ via $p_0(a)(1-S_a(t))\widehat\lambda_t(a)=-\dot m_t(a)$. The first term equals $\E_{Y\sim p_t^{\mathrm{ac}}}[\sum_b\widehat\lambda_t(b)P_\theta(b\given Y,t)]$ since $\sum_a p_0(a)(1-S_a(t))q_t(y\given a)=p_t^{\mathrm{ac}}(y)$; the third sums to $-\sum_b(-\dot m_t(b))$, and the tower identity $\E_{Y\sim p_t^{\mathrm{ac}}}[\sum_b\widehat\lambda_t(b)p_t(b\given Y)]=\sum_b(-\dot m_t(b))$ combines the two into $\mathcal{L}_{\mathrm{RB}}$. The diffusion term is $\mathcal{L}_{\mathrm{score}}$.
\end{proof}

$\mathcal{L}_{\mathrm{CE}}$ is the SJD analogue of the MDLM cross-entropy bound, with the absorbing point replaced by a VP-Gaussian draw and the time-only weight by the per-anchor $-\dot m_t(a)$. $\mathcal{L}_{\mathrm{RB}}$ is the rate-balancing part of the jump term, and its role is sharper than self-consistency: $\mathcal{L}_{\mathrm{CE}}$ alone is minimized by the allocation $\pi^\star_t(a\given y)\propto\widehat\lambda_t(a)p_t(a\given y)$ of \eqref{eq:plugin-allocation} rather than by the posterior, and $\mathcal{L}_{\mathrm{RB}}$, linear in $P_\theta$, is exactly the correction that restores the posterior as the joint minimizer. $\mathcal{L}_{\mathrm{score}}$ is conditional denoising score matching; because the induced score of \Cref{app:classifier-score} is exact at the cross-entropy optimum, we train no score head and optimize $\mathcal{L}_{\mathrm{CE}}+\mathcal{L}_{\mathrm{RB}}$ alone. Every $w$-dependence is reparametrizable: unlike GenMD4, whose discrete forward pass forces a REINFORCE estimator \citep[App.~I.2]{shi2024md4}, the continuous off-anchor state yields an exact $\partial_w$ from one backward pass. In the degenerate-kernel limit of \Cref{sec:conn-discrete}, the score term vanishes and the jump contribution $\mathcal{L}_{\mathrm{CE}}+\mathcal{L}_{\mathrm{RB}}$ reduces to the GenMD4 negative ELBO under $\alpha_{t,a}\leftrightarrow S_a(t)$. For an anchor-agnostic schedule, $\mathcal{L}_{\mathrm{RB}}\equiv 0$.

\paragraph{The objective cannot identify a useful hazard.} First, the only term that couples $P_\theta$ to the choice of hazard carries no signal at the optimum:
\begin{lemma}[Rate-balance vanishes at the optimum]\label{lem:rb-vanishes}
If $P_\theta(\cdot\given y,t)=p_t(\cdot\given y)$ for almost every $(y,t)$, then $\mathcal{L}_{\mathrm{RB}}(\theta,w)=0$ for every $w$.
\end{lemma}
\begin{proof}
At the Bayes classifier the integrand of \eqref{eq:lrb} is zero pointwise, for every $w$.
\end{proof}
\noindent At $w\equiv1$ it moreover vanishes identically in $\theta$, since $\widehat\lambda_t$ is anchor-agnostic. Second, the full bound is flat in $w$ at the per-$w$ optimum. At the Bayes classifier the plug-in rates are exact (\Cref{prop:plugin-consistent}) and the induced score is exact (\Cref{app:classifier-score}), so the model is the exact time reversal of the forward law for \emph{every} $w$: the bound is tight at the data entropy, $\mathcal{L}(\theta^\star(w),w)=\mathrm{H}(p_0)$ in the well-specified limit---the operational form of the schedule non-identifiability of \citet{kingma2021variational}. Since $\theta^\star$ is also a stationary point of $\mathcal{L}$ in $\theta$, the envelope theorem gives $\partial_w\mathcal{L}\big|_{\theta^\star}=0$, hence
\[
\partial_w\big(\mathcal{L}_{\mathrm{CE}}+\mathcal{L}_{\mathrm{RB}}\big)\big|_{\theta^\star}=-\partial_w\mathcal{L}_{\mathrm{score}}\big|_{\theta^\star}:
\]
at convergence, the trained objective's hazard gradient is exactly the negative of the omitted score term's. Descending it inflates $w$ fastest on anchors carrying the most $p_0$-weighted score residual (the prefactor $p_0(a)(1-S_a(t))$ grows with $w(a)$ and is largest for common anchors), i.e.\ a frequency-ordered schedule; away from the optimum, the residual $w$-signal in $\mathcal{L}_{\mathrm{RB}}$ supplies the opposite pull. Both directions materialize below, and both degrade generation when deployed.

\section{Implementation details}\label{app:implementation-details}
\subsection{Convolution closure for the VP-matched SJD}\label{app:convolution}
The DHM target \eqref{eq:dhm-target} and the plug-in quadrature \eqref{eq:plugin-quadrature} both require the off-anchor conditional $p_t^{\mathrm{ac}}(y\given \vec X_0, i)$ at each position $i$, which is built from the convolution
\[
(r_\tau * K_{\tau\to t})(y\given \vec X_0, i) = \int_{\R^d} r_\tau(z\given \vec X_0, i) K_{\tau\to t}(y\given z)\dd z,
\]
of the unsticking kernel at time $\tau$ with the forward diffusion semigroup from $\tau$ to $t$. We work out this convolution explicitly for the VP-matched family of \Cref{ex:vp-matched}, where both factors are Gaussian and the convolution stays Gaussian. The per-anchor case is recovered at $W = I$, where $\mu_i(\vec X_0) = E(a^{(i)})$ depends only on the local anchor and the conditioning collapses to $p_t^{\mathrm{ac}}(y\given a^{(i)})$.

\paragraph{The VP semigroup.}
The VP-SDE $\dd \vec X_t = -\tfrac12 \beta(t) \vec X_t \dd t + \sqrt{\beta(t)}\dd \vec W_t$ has the linear-Gaussian transition kernel
\[
K_{\tau\to t}(\cdot\given z) = \cN\left(\cdot;\ \tfrac{\alpha(t)}{\alpha(\tau)}z,\ \Big(\sigma^2(t) - \tfrac{\alpha^2(t)}{\alpha^2(\tau)}\sigma^2(\tau)\Big) I_d\right),
\]
where $\alpha(t) = \exp\big(-\tfrac12 \int_0^t \beta(s)\dd s\big)$ and $\sigma^2(t) = 1 - \alpha^2(t)$. The VP-matched unsticking kernel from \Cref{ex:vp-matched} is
\[
r_\tau(\cdot\given \vec X_0, i) = \cN\left(\cdot;\ \alpha(\tau)\mu_i(\vec X_0),\ \eta^2 \sigma^2(\tau)I_d\right),
\qquad
\mu_i(\vec X_0) = \big(WE(\vec X_0)\big)_i.
\]

\paragraph{The convolution.}
Both factors are isotropic Gaussian, so the convolution is again isotropic Gaussian with mean and variance given by linearity. Writing $m := \mu_i(\vec X_0)$ for the blended mean at position $i$,
\begin{align*}
\mathrm{mean:}\quad &\tfrac{\alpha(t)}{\alpha(\tau)} \cdot \alpha(\tau)m = \alpha(t)m, \\[3pt]
\mathrm{variance:}\quad &\Big(\tfrac{\alpha(t)}{\alpha(\tau)}\Big)^2 \cdot \eta^2\sigma^2(\tau) + \Big(\sigma^2(t) - \tfrac{\alpha^2(t)}{\alpha^2(\tau)}\sigma^2(\tau)\Big) = \sigma^2(t) - (1-\eta^2)\tfrac{\alpha^2(t)}{\alpha^2(\tau)}\sigma^2(\tau).
\end{align*}
Writing $v_t(\tau) := \sigma^2(t) - (1-\eta^2)\alpha^2(t)\sigma^2(\tau)/\alpha^2(\tau)$, we recover \eqref{eq:vp-matched-convolution}:
\begin{equation}\label{eq:app-vp-matched-convolution}
(r_\tau * K_{\tau\to t})(\cdot\given \vec X_0, i) = \cN\big(\cdot;\ \alpha(t)\mu_i(\vec X_0),\ v_t(\tau) I_d\big).
\end{equation}
The variance $v_t(\tau)$ is identical to the per-anchor case because the variance derivation depends only on the unsticking variance $\eta^2\sigma^2(\tau)$ and the VP semigroup; the cross-position structure $W$ enters only through the mean.

\paragraph{Sanity checks.}

Three limits confirm that \eqref{eq:app-vp-matched-convolution} is well-behaved.

\begin{figure*}[t]
    \centering
    \includegraphics[width=\textwidth]{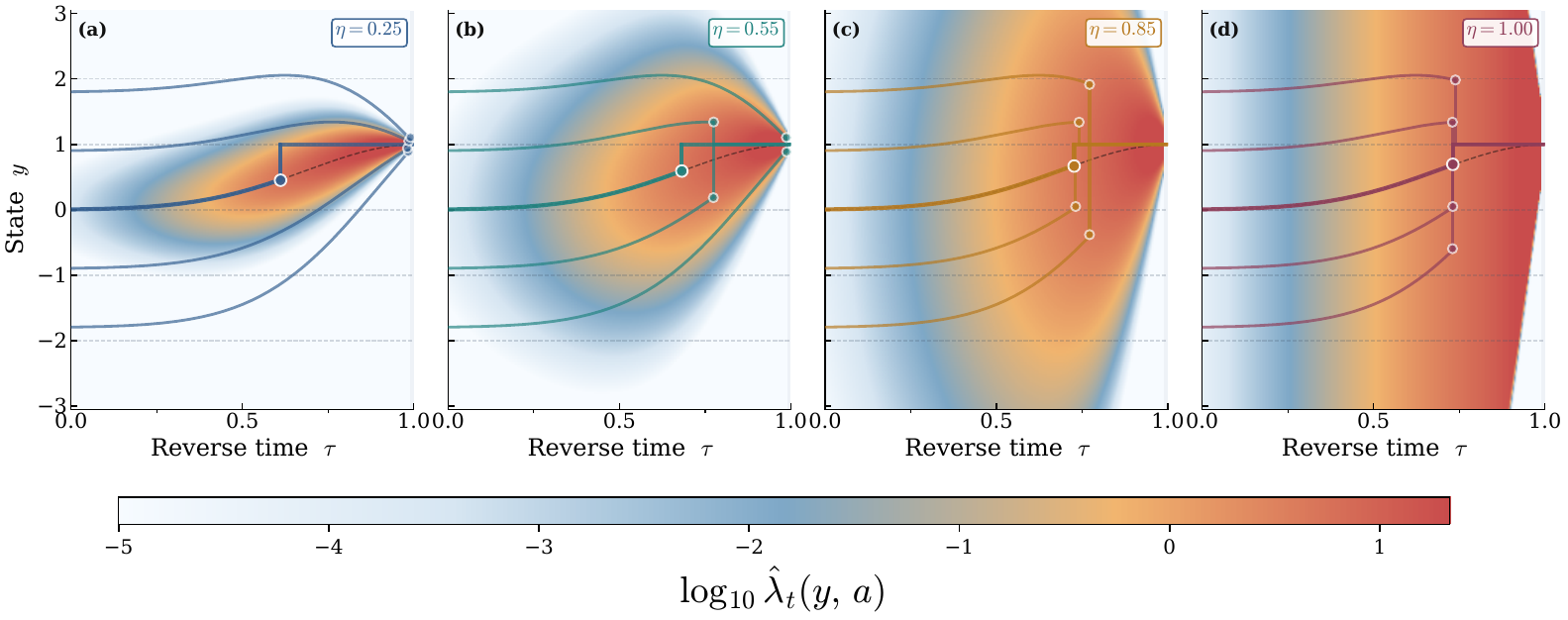}
    \caption{Effect of $\eta$ on the state-dependent hazard target $\widehat{\lambda}_t(y,a)$ under a single anchor $a = 1$ for reverse trajectories (circles mark commitment events). As $\eta \to 1$, the width of the intensity surface widens. As $\eta$ decreases from $1.00$ (d) to $0.25$ (a), the variance of the unsticking kernel sharpens relative to the mixture marginal $p_t^{\mathrm{ac}}$. High intensity regions contract toward the anchor mean, and trajectories in low intensity regions tend to later commitments.}
    \label{fig:eta-commit}
\end{figure*}

\emph{Recovery of the diffusion kernel at $\eta = 1$.} When $\eta = 1$ the unsticking kernel coincides with $q_\tau(\cdot\given a) = K_{0\to\tau}(\cdot\given a)$, and the semigroup property of the VP transition kernel gives $q_\tau * K_{\tau\to t} = q_t$. Consistently, $v_t(\tau)\big|_{\eta = 1} = \sigma^2(t)$ for every $\tau$, recovering the VP perturbation kernel.

\emph{Variance is positive and decreasing in $\tau$ for $\eta < 1$.} For fixed $t$ and $\eta \in (0, 1]$, $v_t(\tau)$ is strictly decreasing in $\tau \in (0, t)$, with $v_t(\tau) \to \eta^2 \sigma^2(t)$ as $\tau \to t^-$ (a particle that just unstuck) and $v_t(\tau) \to \sigma^2(t)$ as $\tau \to 0^+$ (an early unstick has had the full VP semigroup to wash out the variance deficit). In particular $v_t(\tau) > 0$ throughout.

\emph{Terminal Gaussian prior.} As $t \to T$ with $\alpha(T) \to 0$, every mixture component converges to $\cN(0, I_d)$ regardless of $\tau$ and $\eta$, recovering \Cref{rem:prior}.

\subsection{Efficient evaluation of the plug-in hazard}\label{app:plugin}
We give the implementation details behind the cost estimate of the plugin hazard: a one-dimensional quadrature for $p_t^{\mathrm{ac}}(y\given a)$ and the resulting plug-in $\widehat\lambda_t(y, a)$, with a vectorized form that shares structure across anchors. We focus throughout on the VP-matched family of \Cref{ex:vp-matched}, where the convolution \eqref{eq:app-vp-matched-convolution} is closed-form.

\paragraph{Discretizing the quadrature.}
Recall that
\begin{equation}\label{eq:app-plugin-quadrature}
p_t^{\mathrm{ac}}(y\given a) = \int_0^t \vec\lambda_\tau(a) S_a(\tau) (r_\tau * K_{\tau\to t})(y\given a)\dd\tau.
\end{equation}
We discretize $\tau$ on a fixed grid $0 < \tau_1 < \cdots < \tau_{N_\tau} \le t$ with quadrature weights $\{w_i\}_{i=1}^{N_\tau}$ (we use Simpson's rule with $N_\tau \in [16, 64]$ in our experiments). Substituting the closed-form convolution \eqref{eq:app-vp-matched-convolution},
\begin{equation}\label{eq:app-plugin-discretized}
p_t^{\mathrm{ac}}(y\given a) \approx \sum_{i=1}^{N_\tau} w_i \vec\lambda_{\tau_i}(a) S_a(\tau_i) \cN\big(y; \alpha(t) a, v_t(\tau_i) I_d\big).
\end{equation}

\paragraph{Log-domain stability.}
The Gaussian densities in \eqref{eq:app-plugin-discretized} can underflow when $y$ is far from $\alpha(t) a$. We therefore evaluate the log-density by log-sum-exp:
\begin{equation}\label{eq:app-logsumexp}
\log p_t^{\mathrm{ac}}(y\given a) \approx \mathrm{logsumexp}_{i}\Big(\log w_i + \log \vec\lambda_{\tau_i}(a) + \log S_a(\tau_i) + \log \cN\big(y; \alpha(t) a, v_t(\tau_i) I_d\big)\Big).
\end{equation}
The DHM target is then $\log \widehat\lambda_t(y, a) = \log\vec\lambda_t(a) + \log S_a(t) + \log r_a(y, t) - \log p_t^{\mathrm{ac}}(y\given a)$, and the reweighted logits in \eqref{eq:plugin-allocation} become
\[
\ell_\theta(a; y, t) + \log \widehat\lambda_t(y, a),
\]
fed directly into a softmax to sample the destination anchor. 

\paragraph{Sharing structure across anchors.}
For the common case of an anchor-agnostic forward hazard $\vec\lambda_\tau(a) \equiv \vec\lambda(\tau)$, the prefactor $w_i \vec\lambda(\tau_i) S(\tau_i)$ in \eqref{eq:app-plugin-discretized} is independent of $a$ and can be precomputed once per reverse step. The only $a$-dependence in the integrand is in the Gaussian mean $\alpha(t) a$ (recall that the variance $v_t(\tau_i)$ depends only on $\tau_i$ and $t$, not on $a$; a consequence of VP-matching). The quadrature therefore reduces to evaluating $|\cA| \cdot N_\tau$ Gaussian log-densities and one log-sum-exp per anchor. More generally, when the hazard factors as $\vec\lambda_\tau(a) = \beta(\tau) w(a)$, the $\tau$-weights $w_i \beta(\tau_i) S(\tau_i)^{w(a)}$ can be precomputed for each anchor.

\paragraph{Cost.}
Per reverse step and per site $y$, evaluating the full plug-in has three components:
\begin{enumerate}
\item Precomputing the $\tau$-weights and variances: $\mathcal O(N_\tau)$.
\item Evaluating $\cN(y;\alpha(t) a, v_t(\tau_i) I_d)$ for all $(a, i)$: $\mathcal O(N_\tau |\cA| d)$.
\item Log-sum-exp across $\tau_i$ and softmax across $a$: $\mathcal O(N_\tau |\cA|)$.
\end{enumerate}
The dominant term is step 2. With $N_\tau \in [16, 64]$, $|\cA|$ on the order of a vocabulary size, and $d$ on the order of a feature dimension, this is one to two orders of magnitude cheaper than a single forward pass through $P_\theta$ in our experiments. The quadrature is fully vectorizable and runs alongside the classifier on the same accelerator.

\paragraph{Cost of the blending step.}
For the non-local kernel of \Cref{ex:vp-matched}, computing the per-position blended mean $\mu = WE(\vec X_0)$ adds an $L\times L$ matrix multiply with cost $\mathcal O(L^2 d)$ for dense $W$ and $\mathcal O(Ldw)$ for banded $W$ of bandwidth $w$. Both are negligible against one classifier forward pass, and the matmul is computed once per minibatch (training) or once per reverse step (sampling).

\subsection{Classifier-induced score}\label{app:classifier-score}

The reverse drift \eqref{eq:score-drift} requires the score $\nabla \log p_t(y)$ on $\mathsf X_\cA$. We show how to read it off from the trained classifier $P_\theta$ and the same plug-in quadrature used for the hazard, without training a second network \cite{dieleman2022cdcd}.

\paragraph{Setup.}
The marginal density on $\mathsf X_\cA$ factors as $p_t(y) = \sum_{a\in\cA} p_0(a) p_t^{\mathrm{ac}}(y\given a)$, so by the standard mixture-score identity
\begin{equation}\label{eq:app-mixture-score}
\nabla \log p_t(y) = \sum_{a\in\cA} p_t(a\given y) \nabla_y \log p_t^{\mathrm{ac}}(y\given a).
\end{equation}
The classifier $P_\theta$ approximates $p_t(a\given y)$ from \eqref{eq:denoising-posterior}, so the only remaining piece is $\nabla_y \log p_t^{\mathrm{ac}}(y\given a)$.

\paragraph{Score of the mixture conditional.}
Differentiating the integral form of $p_t^{\mathrm{ac}}(y\given a)$ in \eqref{eq:app-plugin-quadrature} under the integral sign,
\[
\nabla_y p_t^{\mathrm{ac}}(y\given a) = \int_0^t \vec\lambda_\tau(a) S_a(\tau) \nabla_y (r_\tau * K_{\tau\to t})(y\given a)\dd\tau,
\]
and dividing by $p_t^{\mathrm{ac}}(y\given a)$,
\begin{equation}\label{eq:app-conditional-score}
\nabla_y \log p_t^{\mathrm{ac}}(y\given a) = \int_0^t \mathrm{w}_t(\tau\given y, a) \nabla_y \log (r_\tau * K_{\tau\to t})(y\given a)\dd\tau,
\end{equation}
where the time-posterior weight is
\begin{equation}\label{eq:app-time-posterior}
\mathrm{w}_t(\tau\given y, a) := \frac{\vec\lambda_\tau(a) S_a(\tau) (r_\tau * K_{\tau\to t})(y\given a)}{p_t^{\mathrm{ac}}(y\given a)}.
\end{equation}
This is the conditional density of the unstick time given $(\vec X_0 = a, \vec X_t = y)$; \eqref{eq:app-conditional-score} expresses the score of the mixture as the average of the component scores, weighted by the time-posterior.

\paragraph{VP-matched closed form.}
Under \Cref{ex:vp-matched}, the convolution \eqref{eq:app-vp-matched-convolution} is Gaussian, so its score is
\[
\nabla_y \log (r_\tau * K_{\tau\to t})(y\given a) = -\frac{y - \alpha(t) a}{v_t(\tau)}.
\]
Substituting into \eqref{eq:app-conditional-score} and \eqref{eq:app-mixture-score},
\begin{equation}\label{eq:app-classifier-score}
s_\theta(y, t) = -\sum_{a\in\cA} P_\theta(a\given y, t) \int_0^t \mathrm{w}_t(\tau\given y, a) \frac{y - \alpha(t) a}{v_t(\tau)}\dd\tau,
\end{equation}
where we have substituted $P_\theta$ for the true posterior $p_t(a\given y)$. \Cref{eq:app-classifier-score} is the classifier-induced score: it depends on $P_\theta$, the precomputed $\tau$-grid quadrature for $\mathrm{w}_t$, and the affine residual $y - \alpha(t) a$.

\paragraph{Sharing computation with the plug-in hazard.}
The time-posterior weight $\mathrm{w}_t(\tau_i\given y, a)$ is the same Gaussian-density times hazard-survival product that appears in the discretized quadrature \eqref{eq:app-plugin-discretized}, normalized by $p_t^{\mathrm{ac}}(y\given a)$. In implementation, the score and the hazard reuse the same set of evaluations: at each $(y, t)$ we compute $\log \cN(y; \alpha(t) a, v_t(\tau_i) I_d)$ once for each $(a, i)$ pair, log-sum-exp over $i$ to obtain $\log p_t^{\mathrm{ac}}(y\given a)$ for the hazard, and softmax over $i$ within each $a$ to obtain $\mathrm{w}_t(\tau_i\given y, a)$ for the score. The marginal cost of the score on top of the hazard is one extra weighted sum per anchor: still negligible compared to a forward pass through $P_\theta$.

\paragraph{Limit at $\eta = 1$.}
For $\eta = 1$, the convolution \eqref{eq:app-vp-matched-convolution} collapses to the VP perturbation kernel $q_t(\cdot\given a)$ for every $\tau$, so the time-posterior $\mathrm{w}_t(\tau\given y, a)$ becomes uniform and integrates out. \Cref{eq:app-classifier-score} reduces to
\[
s_\theta(y, t) \big|_{\eta = 1} = -\sum_{a\in\cA} P_\theta(a\given y, t) \frac{y - \alpha(t) a}{\sigma^2(t)},
\]
which is the standard mixture-of-Gaussians score familiar from masked diffusion with continuous augmentation \cite{zheng2025cadd}. The general $\eta < 1$ formula \eqref{eq:app-classifier-score} interpolates between this masked-diffusion-style score and the per-component VP score by averaging over the unstick time.

This $(W,\eta)=(I,1)$ limit is exactly the score-interpolation estimator of \citet{dieleman2022cdcd}, and \citet{eijkelboom2024vfm} uses the same construction. SJD contributes nothing new on the drift side at this point. What SJD adds is the jump side: Theorem~\ref{thm:anchored-time-reversal} fixes the commit rate and the destination by flux balance, and Theorem~\ref{thm:dhm} together with Lemma~\ref{lem:per-anchor} makes both estimable from the same classifier. Neither construction has an analogue in the continuous setting, where the process has no jumps. Equation~\eqref{eq:blended-score} below extends the estimator to $W \neq I$.

\paragraph{Blended kernels ($W \neq I$).}
For the non-local kernel of \Cref{ex:vp-matched} the conditioning object is the whole clean sequence, and \eqref{eq:app-classifier-score} does not apply as written: the mixture component at position $i$ is centered at the blended $\alpha(t)\mu_i(\vec X_0)$ of \eqref{eq:blended-mean}, not at $\alpha(t)a^{(i)}$. At $\eta = 1$ the correction is exact and closed-form. Write $y = (y_1, \dots, y_L)$ for the sequence state. Conditionally on $\vec X_0$, the active (off-anchor) coordinates are independent with $y_i \sim \cN\big(\alpha(t)\mu_i(\vec X_0), \sigma^2(t) I_d\big)$, by \eqref{eq:app-vp-matched-convolution} at $v_t(\tau) = \sigma^2(t)$, and a committed position reveals its anchor exactly, since anchors are absorbing. The tower argument of \eqref{eq:app-mixture-score}, now conditioning on the full state (active coordinates and committed anchors), gives
\[
\nabla_{y_i} \log p_t(y) = -\frac{y_i - \alpha(t)\E\big[\mu_i(\vec X_0)\given y\big]}{\sigma^2(t)},
\]
and because $\mu_i$ is linear in the embedded sequence \eqref{eq:blended-mean}, the conditional mean requires only the \emph{per-position marginal} posteriors: $\E[\mu_i(\vec X_0)\given y] = \sum_j W_{ij} m_j(y)$, where $m_j(y) := \sum_{a\in\cA} p_t^{(j)}(a\given y) E(a)$ is the posterior-mean embedding at position $j$ and $p_t^{(j)}(\cdot\given y)$ the marginal posterior of that position's anchor given the full state; at committed and clamped positions, $m_j = E(a^{(j)})$ exactly. Substituting the classifier yields the blended-mean induced score
\begin{equation}\label{eq:blended-score}
s_{\theta,i}(y, t)\big|_{\eta=1} = -\frac{y_i - \alpha(t)\big(W m_\theta(y, t)\big)_i}{\sigma^2(t)},
\qquad
m_{\theta,j}(y, t) = \sum_{a\in\cA} P_\theta^{(j)}(a\given y, t) E(a).
\end{equation}
Three remarks. First, \eqref{eq:blended-score} is exact, not a mean-field approximation: the conditional score is linear in $\mu_i$, so the joint posterior over sequences is never needed, and the single per-position classifier remains sufficient at $W \neq I$. Second, the jump side needs no modification: at $\eta = 1$ the spatial factor of the DHM target collapses to $1/(1 - S_a(t))$ for \emph{any} $W$ so the commit rates and allocation are unchanged, and \eqref{eq:blended-score} is the only place $W$ enters the sampler: one blending product $W m_\theta$ per reverse step, the same product as in training (\Cref{app:plugin}). Third, at $W = I$ it reduces to the $\eta = 1$ limit above. We use \eqref{eq:blended-score} whenever $W \neq I$; for $\eta < 1$ the time-posterior \eqref{eq:app-time-posterior} depends on $\mu_i(\vec X_0)$ inside the conditional expectation, per-position marginals no longer suffice, and we do not combine $W \neq I$ with $\eta < 1$.

\subsection{Sampling}\label{app:sampling}

A trained classifier $P_\theta$ from \Cref{sec:hazard} fully specifies an approximation to the reverse-time SJD of \Cref{thm:anchored-time-reversal}: the score and the per-anchor reverse intensities are both available through the plug-in~\eqref{eq:plugin-estimators}. We assemble these into a sampler that runs from the prior $p_T$ at $\tau = 0$ to data at $\tau = T$.

The classifier also receives no explicit commitment indicator. Its input is the state alone. A committed coordinate is clamped exactly at its anchor, and a diffusing coordinate hits an anchor with probability zero, so the state determines the commitment pattern almost surely. At $(W,\eta)=(I,1)$ the state is a measurable bijection of the explicit pair $(x_t,z_t)$ used by hybrid models with a separate discrete channel (Proposition~\ref{prop:plain-hybrid}(i)), so no information is lost by dropping that channel.

\paragraph{Predictor: one reverse step.}
At reverse time $\tau$ (forward time $t = T - \tau$) and current state $y \in \mathsf X_\cA$, one Euler step of size $\Delta\tau$ does two things in parallel:
\begin{enumerate}
    \item \emph{Drift and diffuse along the score.} Compute the classifier-induced score $s_\theta(y, t)$ (\Cref{app:classifier-score}) and advance the continuous SDE component via a standard Euler--Maruyama step.
    \item \emph{Sample a sticky jump.} With probability $1 - \exp\big(-\lambda_{\mathrm{plug},\theta}(y, t)\Delta\tau\big)$, commit to an anchor drawn from $\pi_{\mathrm{plug},\theta}(\cdot\given y, t)$ \eqref{eq:plugin-allocation} and treat that anchor as absorbing for the rest of the reverse trajectory.
\end{enumerate}
Both steps reuse the same forward pass through $P_\theta$: the classifier produces logits over anchors, from which the score and the reweighted jump distribution are read off. \Cref{alg:sjd-sampling} states the full sampler.

\paragraph{Evaluating the plug-in hazard.}
The plug-in \eqref{eq:plugin-estimators} requires the mixture density $p_t^{\mathrm{ac}}(y\given a)$ pointwise, evaluated by
\begin{equation}\label{eq:plugin-quadrature}
p_t^{\mathrm{ac}}(y\given a) = \int_0^t \vec\lambda_\tau(a) S_a(\tau) (r_\tau * K_{\tau\to t})(y\given a) \dd\tau,
\end{equation}
which is one-dimensional and admits a closed-form integrand whenever the simulation-free conditions of \Cref{sec:sim-free} hold. We discretize \eqref{eq:plugin-quadrature} on a $\tau$-grid of $N_\tau$ points and evaluate the log of the integrand by log-sum-exp for numerical stability. The cost is negligible compared to a forward pass through $P_\theta$. See \Cref{app:plugin} for details.

\section{Experimental details}\label{app:experiments-details}

\subsection{CIFAR-10}\label{app:cifar10-imagenet}
\paragraph{Value-space blending kernel.}
The blended arms of \Cref{fig:cifar10_sigmaB} replace the linear embedding blend \eqref{eq:blended-mean} with a blend applied to raw pixel values before embedding. Flatten the $32\times32$ pixel grid row-major into $N=1024$ sites, indexed by $i\mapsto(r_i,c_i)$ with $r_i=\lfloor i/32\rfloor$ and $c_i = i\bmod 32$, and write the squared grid distance
\begin{equation}
  d^2_{ij} = (r_i-r_j)^2 + (c_i-c_j)^2 .
\end{equation}
The blending matrix $B^{(\sigma_B)}\in\mathbb{R}^{N\times N}$ is the row-stochastic Gaussian position kernel
\begin{equation}
  B^{(\sigma_B)}_{ij} = \frac{\exp\big(-d^2_{ij}/2\sigma_B^{2}\big)} {\sum_{k=1}^{N}\exp\big(-d^2_{ik}/2\sigma_B^{2}\big)},
  \qquad \sigma_B\in\{1.5,2.0,2.5\},
\end{equation}
equivalently the row-wise softmax of $-D/2\sigma_B^2$ with $D=[d^2_{ij}]$. Row-stochasticity keeps blended values on the pixel scale, and $B^{(\sigma_B)}\to I$ as $\sigma_B\to 0$, recovering the identity-kernel hybrid. With $v(\vec X_0)\in[0,255]^{N\times 3}$ the raw pixel values and $\mathcal{E}$ the fixed sinusoidal embedding of the value scale,
\begin{equation}
  \mathcal{E}(v) = \Big[\sin\big(2\pi f v/P\big), \cos\big(2\pi f v/P\big)\Big]_{f=1}^{F},
  \qquad F=4,\quad P=255,\quad \lVert\mathcal{E}(v)\rVert_2=\sqrt{F}=2,
\end{equation}
the forward mean blurs the raw values and embeds the blurred result,
\begin{equation}\label{eq:embed-the-blur}
  \mu(\vec X_0) =\ \mathcal{E}\big(B^{(\sigma_B)} v(\vec X_0)\big),
\end{equation}
with $B^{(\sigma_B)}$ acting over the flattened spatial grid independently for each of the $C=3$ colour channels; the token sequence has length $L=3N=3072$, one \texttt{uint8} token per channel per site. The unsticking kernel \eqref{eq:nonlocal-r} uses this mean at $\eta=1$. The composition $\mathcal{E}\circ B^{(\sigma_B)}$ is smooth in the pixel values, which is what makes natural images low-pass under this corruption; the same blur applied to arbitrary learned embeddings does not preserve that structure.

At sampling, the induced score and plug-in weights require the forward mean under the denoising posterior. We form the per-site posterior-mean value $\hat v_i = \sum_{v=0}^{255} P_\theta(v\given\cdot)v$, blur it through $B^{(\sigma_B)}$, and embed the blended result, $\mathcal{E}\big(B^{(\sigma_B)}\hat v\big)$; the embedding step is the mean-field substitution of \Cref{app:classifier-score}.

\paragraph{Setup and common hyperparameters.}
All CIFAR-10 experiments used unconditional generation on images represented as discrete \texttt{uint8} tokens of shape \(32\times 32\times 3\) with vocabulary size \(256\) \citep{shi2024md4,zheng2025cadd}. We used no data augmentation. Unless otherwise noted, training and validation used batch size \(512\), AdamW with learning rate \(10^{-4}\), \(\beta_2=0.999\), weight decay \(0\), and a 1{,}000-step warmup followed by a constant learning-rate schedule. All runs used seed \(0\), were trained for \(200{,}000\) steps, used EMA decay \(0.9999\), and saved checkpoints every \(10{,}000\) steps; the best checkpoint was selected by validation FID. We report FID-50k against CIFAR-10 train statistics, using FID evaluation batch size \(128\). All samplers used \(256\) NFEs.

All CIFAR-10 models used ADM-family U-Net backbones \citep{dhariwal2021adm}. SJD, MD4, MDLM, and Bit Diffusion used the 5D image-token ADM variant, while DDPM used the 2D ADM variant. Shared ADM hyperparameters were feature dimension \(96\), \(32\) layers, dropout \(0.1\), channel multipliers \([3,4,4]\), \(4\) residual blocks per level, and attention at resolutions \([2,4]\). We train using two NVIDIA H100 GPUs. Owing to compute constraints, each CIFAR-10 configuration is trained with a single seed, following the single-run reporting practice of prior diffusion evaluations on this benchmark \citep{ho2020ddpm,chen2023analogbits,shi2024md4,zheng2025cadd}, so FID differences on the order of run-to-run variability should be read with caution.

\paragraph{Model-specific settings.}
SJD used a linear VP schedule with \(\beta_{\min}=0.1\), \(\beta_{\max}=20.0\), and \(T=1\), a polynomial-\(\alpha\) hazard with \(p=1\). 
%sampling used a cosine grid; anchors are learned end-to-end and normalized following the methodology from \citet[Section 3.2]{dieleman2022cdcd}. 
MD4 used the continuous-time masked discrete diffusion setup with a linear noise schedule, antithetic time sampling, ancestral sampling, a cosine grid, and top-\(p=0.98\). DDPM used a linear \(\beta\) schedule from \(10^{-4}\) to \(2\times 10^{-2}\) and \(\epsilon\)-prediction. MDLM used a continuous-time masked discrete model with a linear noise schedule, ancestral sampling, a cosine grid, and no prediction caching. Bit Diffusion used \(x_0\)-prediction, an MSE loss, a linear signal schedule, and deterministic DDIM-style sampling on a uniform grid. In our experiments Bit Diffusion is not self-conditioned.

\subsection{Text8}\label{app:text8}
We follow the setup of \citet{lou2024sedd, austin2021d3pm} on 256-character sequences with an 8M-parameter DiT denoiser. The blending matrix $W$ is a 1D Gaussian convolution of bandwidth $\sigma_W$, swept over $\sigma_W \in \{0.5, 1.0, 1.5, 2.0, 2.5\}$; $W = I$ ($\sigma_W = 0$) recovers the CANDI baseline.

\paragraph{Noise schedule.}
All methods use the log-linear schedule for the discrete masking corruption \citep{sahoo2024mdlm, pynadath2025candi}. CANDI and SJD additionally use the CANDI rank-degradation schedule with $r_{\min} = 0.05$ and $r_{\max} = 0.4$.

\paragraph{Training.}
\Cref{tab:text8-hparams} lists the optimizer and batch settings, shared across all methods.

\begin{table}[h]
\centering
\caption{Text8 training hyperparameters.}
\label{tab:text8-hparams}
\begin{minipage}[t]{0.4\textwidth}
\centering
\begin{tabular}{lc}
\toprule
Hyperparameter & Value \\
\midrule
\multicolumn{2}{l}{\textit{Optimizer}} \\
Algorithm & AdamW \\
Learning rate & $1 \times 10^{-3}$ \\
$\beta_1$ & 0.9 \\
$\beta_2$ & 0.999 \\
$\epsilon$ & $1 \times 10^{-8}$ \\
Weight decay & 0 \\
Gradient clip & 1.0 \\
\bottomrule
\end{tabular}
\end{minipage}%
\hfill
\begin{minipage}[t]{0.6\textwidth}
\centering
\begin{tabular}{lc}
\toprule
Hyperparameter & Value \\
\midrule
\multicolumn{2}{l}{\textit{Batch size}} \\
Global batch size & 512 \\
Gradient accumulation steps & 4 \\
\midrule
\multicolumn{2}{l}{\textit{Training}} \\
Max steps & 40{,}000 \\
Precision & \texttt{bf16} \\
EMA decay & 0.9999 \\
Antithetic time sampling & \checkmark \\
LR scheduler & Constant \\
Hardware & 1 RTX A6000 \\
\bottomrule
\end{tabular}
\end{minipage}
\end{table}

\paragraph{Inference and quality--diversity frontier.}
MDLM uses ancestral sampling. CANDI and SJD use the exact hybrid sampler of \citet{pynadath2025candi}. For each method and each NFE budget we sweep the sampling temperature over 11 evenly spaced values in $[0.05, 1.0]$ to trace the quality--diversity frontier (\Cref{fig:text8-frontier}); \Cref{fig:text8} reports the maximum valid-word count along each frontier. SJD at $\sigma_W \geq 1.0$ Pareto-dominates CANDI and MDLM throughout the high-diversity regime, with the dominance widening at higher NFE.

\begin{figure}
    \centering
    \includegraphics[width=\linewidth]{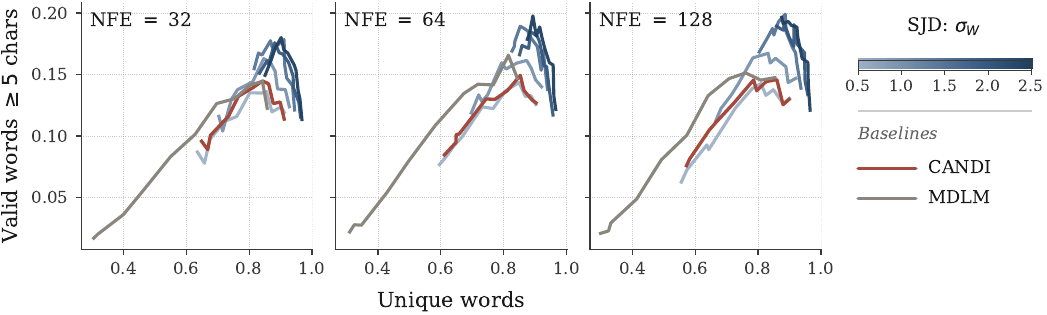}
    \caption{Quality--diversity frontiers on Text8, by NFE budget. Each curve is a temperature sweep in the (unique-word fraction, valid-word fraction at length $\geq 5$) plane; upper-right dominates. The SJD family at $\sigma_W \geq 1.0$ Pareto-dominates CANDI and MDLM at high diversity.}
    \label{fig:text8-frontier}
\end{figure}

\subsubsection{Intermediate commitment versus terminal decoding}
\label{app:terminal-decoding}

We compare SJD's intermediate commitment with terminal nearest-anchor decoding on Text8. In the terminal-decoding baseline, each final continuous state $y$ is mapped to a token by
\[
    \hat a(y) = \arg\min_{a\in\mathcal A}
    \big\|y-E(a)\big\|_2.
\]
The comparison matches the classifier, token embeddings, and training budget. We sample the continuous baseline using both the reverse SDE and the probability-flow ODE (PF-ODE). Whereas SJD commits to anchors during generation, the continuous baselines defer all discrete decisions to the terminal rounding step. \Cref{tab:terminal-decoding} reports valid-word fractions at matched numbers of function evaluations (NFE).

\begin{table}[htbp]
    \centering
    \caption{Intermediate commitment versus terminal decoding on Text8. Higher valid-word fractions are better. The classifier, embeddings, and training budget are matched; the continuous baselines use terminal nearest-anchor rounding.}
    \label{tab:terminal-decoding}
    \small
    \begin{tabular}{@{}lcccc@{}}
        \toprule
        & \multicolumn{3}{c}{Word length $\geq 5$}
        & Word length $\geq 6$ \\
        \cmidrule(lr){2-4}\cmidrule(l){5-5}
        Sampler & 32 NFE & 64 NFE & 128 NFE & 128 NFE \\
        \midrule
        Continuous reverse SDE & 0.081 & 0.102 & 0.117 & 0.055 \\
        Continuous PF-ODE & 0.078 & 0.088 & 0.092 & 0.039 \\
        SJD & \textbf{0.180} & \textbf{0.193}
            & \textbf{0.212} & \textbf{0.130} \\
        \bottomrule
    \end{tabular}
\end{table}

At the five-character threshold, SJD attains $1.81$--$2.31$ times the valid-word fraction of the continuous baselines across the tested budgets. At 128 NFE, its relative advantage increases at the six-character threshold, reaching factors of $2.36$ over the reverse SDE and $3.33$ over the PF-ODE. Moreover, SJD at 32 NFE achieves $0.180$, exceeding both continuous baselines at 128 NFE ($0.117$ and $0.092$). Thus, in this comparison, intermediate commitment yields higher measured validity with one quarter of the function evaluations.

\subsubsection{Likelihood-bound comparison with MD4}
\label{app:md4-likelihood}

We compare SJD with MD4 \citep{shi2024md4} under the Text8 protocol of \Cref{app:text8}, using an 8M-parameter DiT, a matched discrete noise schedule and training budget, and three training seeds per method.
\Cref{tab:md4-likelihood} reports likelihood bounds alongside valid-word generation at 128 NFE.

\paragraph{Training objective and evaluated bound.}
Under the assumptions of \Cref{prop:elbo-eta1}, the SJD negative ELBO in
\Cref{app:e2e-hazard} decomposes as
\[
    \mathcal L = \mathcal L_{\mathrm{CE}}
               + \mathcal L_{\mathrm{RB}}
               + \mathcal L_{\mathrm{score}}.
\]
Here, $\mathcal L_{\mathrm{CE}}$ is cross-entropy weighted by the
per-anchor mass-release rate, $\mathcal L_{\mathrm{RB}}$ is a rate-balancing
correction that vanishes at the Bayes classifier, and
$\mathcal L_{\mathrm{score}}$ is the
conditional score-matching contribution. We train SJD by minimizing
$\mathcal L_{\mathrm{CE}}+\mathcal L_{\mathrm{RB}}$ and obtain its score
from the classifier as in \Cref{app:classifier-score}; MD4 is trained on
its own negative ELBO. The induced score is exact at the Bayes
classifier, but the score contribution is retained when evaluating
the full bound. We denote this SJD training variant by SJD-ELBO.
In the degenerate-kernel limit, the score term disappears and the
jump contribution $\mathcal L_{\mathrm{CE}}+\mathcal L_{\mathrm{RB}}$
reduces to the GenMD4 negative ELBO. For an anchor-agnostic schedule,
$\mathcal L_{\mathrm{RB}}\equiv 0$.

\begin{table}[htbp]
    \centering
    \caption{Text8 likelihood-bound comparison with MD4, using three training seeds per method. BPC denotes bits per character.}
    \label{tab:md4-likelihood}
    \small
    \begin{tabular}{@{}lccc@{}}
        \toprule
        & Full bound & Discrete term & Valid-word fraction \\
        Method & (BPC $\downarrow$) & (BPC $\downarrow$) & (length $\geq 5$, 128 NFE $\uparrow$) \\
        \midrule
        MD4 & $1.968\pm0.005$ & $1.968\pm0.005$ & $0.207\pm0.004$ \\
        SJD-ELBO & $2.122\pm0.010$ & $1.291\pm0.003$ & $0.211\pm0.005$ \\
        \bottomrule
    \end{tabular}
\end{table}

\paragraph{Interpretation.}
The SJD discrete term is $0.677$ BPC below MD4's negative ELBO, but its denoiser also observes a noisy continuous representation of the clean token. This supplies information unavailable at masked positions in MD4, so the discrete term alone does not establish a likelihood improvement. The continuous contribution adds $0.831$ BPC, giving a full bound $0.154$ BPC above MD4's. These full bounds can be reported together, but their ordering does not establish the ordering of the exact model likelihoods because their variational gaps may differ. CADD similarly cautions against directly comparing bound tightness across models and reports its discrete component for reference \citep[App.~C.2]{zheng2025cadd}. CCDD removes information about masked tokens from the continuous input for its main-text validation-perplexity evaluation \citep[App.~D.2]{zhou2025ccdd}.  Finally, SJD-ELBO and MD4 attain similar valid-word fractions at 128 NFE, so the likelihood-bound comparison is accompanied by comparable performance on this generation metric.

\subsubsection{Forward-hazard ablations}
\label{app:text8-hazards}

\paragraph{Learned schedules degrade generation.} We test on Text8 at $W=I$, $\eta=1$, training for $150$k steps and reporting the valid-word fraction at length $\geq5$, maximized over the temperature/NFE sweep of \Cref{app:text8}. A learned hazard can hurt through two routes: at training, the weight $-\dot m_t(a)$ skews which $(a,t)$ pairs $P_\theta$ sees; at inference, $\widehat\lambda_t(a)$ reweights the allocation \eqref{eq:plugin-allocation}. An offline read-out that overrides $w\equiv1$ in the sampler only separates them. In \Cref{tab:e2e-text8}, the \emph{bounded} variant ($w\in[0.5,2]$, cross-entropy driven) learns the frequency-ordered schedule predicted by the envelope identity; the \emph{uniform-CE} variant ($w\equiv1$ inside the weight, so $\partial_w\mathcal{L}_{\mathrm{CE}}=0$ and the residual signal is $\partial_w\mathcal{L}_{\mathrm{RB}}$) runs anti-frequency to its clamp $[0.05,20]$. Both classifiers stay healthy, forcing $w\equiv1$ at sampling recovers $0.10$--$0.14$, at the fixed-hazard ELBO reference of $0.13$, yet the deployed hazard degrades generation monotonically in spread: $1\times\to0.13$, $4\times\to0.066$, $403\times\to0.014$, against $0.27$ for the plain-CE baseline within the same budget ($0.28$ at $300$k).
Note that the ELBO weighting alone already costs half the metric: at $w\equiv1$, $\mathcal{L}_{\mathrm{RB}}\equiv0$ and the objective reduces to $-\dot m_t$-weighted cross-entropy, dropping $0.27\to0.13$.

\begin{table}[t]
\centering
\caption{End-to-end hazard learning on Text8 ($W=I$, $\eta=1$): valid-word fraction at length $\geq 5$, maximized over the temperature/NFE sweep of \Cref{app:text8}; best checkpoint within $150$k steps (deployed learned-hazard runs peak at $25$k). ``Deployed'' samples with the trained hazard; ``$w\equiv1$ sampler'' overrides it to uniform at sampling only. Spread is $\max_a w(a)/\min_a w(a)$ at convergence; parentheses give the rank correlation between $\log w(a)$ and character frequency.}
\label{tab:e2e-text8}
\begin{tabular}{llccc}
\toprule
Configuration & $w$ schedule & spread & deployed & $w\equiv1$ sampler \\
\midrule
Fixed hazard, plain CE \eqref{eq:classifier-loss} & uniform               & $1\times$   & $0.27$  & $0.27$ \\
Fixed hazard, ELBO-trained $\theta$               & uniform               & $1\times$   & $0.13$  & $0.13$ \\
\midrule
Bounded (cross-entropy--driven)                   & frequency ($+0.99$)   & $4\times$   & $0.066$ & $0.10$ \\
Uniform-CE ($\mathcal{L}_{\mathrm{RB}}$-driven)   & anti-freq.\ ($-0.7$)  & $403\times$ & $0.014$ & $0.14$ \\
\bottomrule
\end{tabular}
\end{table}

\paragraph{Fixed non-uniform hazards are nearly neutral.} The learned runs confound three factors: the objective's $w$-dynamics, the damage the $w$-dependent CE weight does to $\theta$, and, in the $w\equiv1$ read-outs, a train/inference hazard mismatch. To probe the axis itself we remove all three: we hardcode $w(a)\propto\mathrm{freq}(a)^{\gamma}$, normalized so $\sum_a p_0(a)w(a)=1$, freeze them, and train by plain cross-entropy \eqref{eq:classifier-loss}, so the hazard enters only the corruption and the inference reweighting and the classifier is trained and sampled under the same hazard. 

The exponent $\gamma$ sets spreads of $2\times$ and $4\times$ in each direction: frequency-keyed ($\gamma>0$; common characters commit \emph{last} in reverse time) and anti-frequency ($\gamma<0$; common characters commit \emph{first}). Every arm runs three seeds under the offline protocol of \Cref{app:text8}; as a plumbing check, $\gamma=0$ through this path reproduces the historical fixed-hazard baseline.
\Cref{tab:frozen-hazard} is unambiguous: every schedule lands within $0.02$ of uniform in the mean, and the same $4\times$ frequency direction that collapsed to $0.066$ under the ELBO recipe reaches $0.279$ here. The failure of end-to-end hazard learning is therefore a property of the objective, not of the axis: the bound is flat in $w$ at the optimum, the CE weight—the only strong $w$-signal—corrupts the classifier it trains, and either runaway direction degrades generation through the reverse-sampler allocation.

\begin{table}[t]
\centering
\caption{Frozen frequency-keyed hazards on Text8 (plain cross-entropy \eqref{eq:classifier-loss}, $W=I$, $\eta=1$): valid-word fraction at length $\geq 5$ under the offline protocol of \Cref{app:text8}, mean$\pm$std over three seeds. Anti-frequency schedules commit common characters first in reverse time; frequency-keyed schedules commit them last. No non-uniform schedule clears uniform under the pre-registered mean$-$std criterion at either budget.}
\label{tab:frozen-hazard}
\begin{tabular}{lcc}
\toprule
Hazard & spread & valid-word (len $\geq 5$) \\
\midrule
\multicolumn{3}{@{}l}{\textit{$150$k steps (three seeds)}}\\
\quad anti-frequency       & $4\times$ & $0.262 \pm 0.004$ \\
\quad anti-frequency       & $2\times$ & $0.272 \pm 0.015$ \\
\quad uniform ($w\equiv1$) & $1\times$ & $0.263 \pm 0.000$ \\
\quad frequency            & $2\times$ & $0.279 \pm 0.013$ \\
\quad frequency            & $4\times$ & $0.279 \pm 0.008$ \\
\midrule
\multicolumn{3}{@{}l}{\textit{$300$k steps (three seeds)}}\\
\quad anti-frequency       & $2\times$ & $0.270 \pm 0.035$ \\
\quad uniform ($w\equiv1$) & $1\times$ & $0.273 \pm 0.005$ \\
\bottomrule
\end{tabular}
\end{table}

\subsection{Sudoku}\label{app:sudoku}
We use the Sudoku dataset and split from \citet{shah2024puzzles}. Each instance consists of a 9$\times$9 puzzle (with blanks) and its unique completed solution. The dataset is filtered to puzzles solvable without backtracking using a fixed set of human-style strategies, ensuring that test-time evaluation probes logical deduction rather than combinatorial search.
\begin{table}[h]
\centering
\caption{Sudoku dataset statistics}
\label{tab:sudoku_dataset_stats}
\small
\begin{tabular}{lc}
\toprule
Statistic & Value \\
\midrule
Grid size & $9\times 9$ \\
Vocabulary & digits $\{0,\dots,9\}$ ($|\mathcal V|=10$) \\
Train puzzles & 1{,}800{,}000 \\
Test puzzles & 100{,}000 \\
Filtering & solvable by fixed strategy set \\
URL & \href{https://drive.google.com/drive/folders/1TluiZjYl-zLdbxjVmhfWl-WyX_OvD7UW}{source} \\
\bottomrule
\end{tabular}
\end{table}
The model receives an input such as ($080050060\ldots603100007$) and predicts the full solution ($789251364\ldots653184297$). This matches the MGDM-style Sudoku formulation \citep{ye2025mgdm} in which the puzzle is represented as an 81-token discrete sequence with (0) used as the masking symbol.

\paragraph{Constraint-graph jump kernel.}
The Sudoku board is the $L=81$ token positions of a $9\times 9$ grid. Index a position $i\in\{0,\dots,80\}$ by its row, column, and $3\times 3$ box,
\begin{equation}
r_i = \left\lfloor i/9 \right\rfloor,\qquad
c_i = i \bmod 9,\qquad
b_i = 3\left\lfloor r_i/3 \right\rfloor + \left\lfloor c_i/3 \right\rfloor .
\end{equation}
Two positions are Sudoku-constraint neighbours iff they share a unit (row, column, or box):
\begin{equation}
\mathcal{N}(i) = \{i\}\cup
\underbrace{\{ j : r_j=r_i \}}_{\text{row}}\cup
\underbrace{\{ j : c_j=c_i \}}_{\text{column}}\cup
\underbrace{\{ j : b_j=b_i,\ r_j\neq r_i,\ c_j\neq c_i \}}_{\text{box}} .
\end{equation}
The fixed blending matrix $W\in\mathbb{R}^{81\times 81}$ is the constraint-graph Gaussian with bandwidth $\sigma_W=1.5$,
\begin{equation}
W_{ij} =
\begin{cases}
\exp\left(-\dfrac{(r_i-r_j)^2+(c_i-c_j)^2}{2\sigma_W^2}\right), & j\in\mathcal{N}(i),\\[2ex]
0, & \text{otherwise.}
\end{cases}
\end{equation}
The blended mean \eqref{eq:blended-mean} of each cell is thus its own embedding plus a Gaussian-weighted sum of its $20$ constraint peers,
\begin{equation}
\mu_i(\vec{X}_0) = \big(WE(\vec{X}_0)\big)_i
= E\big(a^{(i)}\big)
+ \sum_{j\in\mathcal{N}(i)\setminus\{i\}}
\exp\left(-\frac{(r_i-r_j)^2+(c_i-c_j)^2}{2\sigma_W^2}\right) E\big(a^{(j)}\big),
\end{equation}
and the unsticking kernel \eqref{eq:nonlocal-r} uses this mean at $\eta=1$,
\(
r_t(y \given i, \vec{X}_0)
= \mathcal{N}\big(y; \alpha(t)\mu_i(\vec{X}_0), \sigma^2(t)I_d\big).
\)

\end{document}